\documentclass{article} 
\usepackage{iclr2027_conference,times}

\usepackage{hyperref}
\usepackage{url}

\usepackage{booktabs}       
\usepackage{amsfonts}       
\usepackage{nicefrac}       
\usepackage{microtype}      
\usepackage{xcolor}         
\usepackage{multirow}
\usepackage{float}
\usepackage{wrapfig}
\usepackage{graphicx}
\usepackage{enumitem}
\usepackage{subcaption}
\usepackage{mathtools} 
\usepackage{amsmath}
\usepackage{amssymb}
\usepackage{amsthm}

\usepackage{algorithm}
\usepackage{algorithmic}

\theoremstyle{plain}
\newtheorem{assumption}{Assumption}
\newtheorem{definition}{Definition}
\newtheorem{condition}{Condition}

\theoremstyle{plain}
\newtheorem{proposition}{Proposition}
\newtheorem{theorem}{Theorem}
\newtheorem{corollary}{Corollary}
\newtheorem{lemma}{Lemma}

\theoremstyle{remark}
\newtheorem{remark}{Remark}

\newcommand{\pa}[1]{\mathrm{Pa}(#1)}
\newcommand{\ch}[1]{\mathrm{Ch}(#1)}
\newcommand{\spn}[1]{\mathrm{Sp}(#1)}

\newcommand{\MB}[1]{\mathrm{MB}(#1)}

\title{Decoupled Causal Discovery}

\author{Zhengkang~Guan,
Fei Wu, 
Kun~Kuang\thanks{Corresponding author}
\\
College of Computer Science and Technology
\\
Zhejiang University
\\
\texttt{zhengkang.guan@zju.edu.cn, kunkuang@zju.edu.cn}
\\
}

\iclrfinalcopy 

\begin{document}

\maketitle

\lhead{\textit{Preprint.}}

\begin{abstract}
Causal discovery from observational data is a fundamental yet challenging task in scientific research.
While existing approaches are primarily based on conditional independence tests, structure scores, or restrictive functional assumptions, we propose Decoupled Causal Discovery (DCD), a novel decoupling-based perspective that does not rely on these methodologies.
DCD directly identifies the Markov boundary (MB) by decoupling non-target variables via weighting functions, such that only variables within the MB preserve dependence with the target under the decoupled distribution.
Building on this, DCD iteratively constructs the Completed Partially Directed Acyclic Graph (CPDAG) by exploiting structural asymmetries within the MBs.
We establish the theoretical identifiability, soundness, and completeness of DCD.
Empirical evaluations demonstrate that DCD achieves strong performance, particularly excelling in challenging noise regimes.
\end{abstract}

\section{Introduction}

Causal discovery aims to uncover underlying causal structures from observational data and serves as a fundamental task in scientific research and decision-making~\citep{spirtes2016causal, glymour2019review}.
Following decades of development, existing approaches have driven remarkable progress in the field and primarily fall into three paradigms:
constraint-based methods~\citep{PC}, which rely on sequential conditional independence tests (CITs)~\citep{PCStable, peters2017elements};
score-based methods~\citep{GES}, which search for a graph maximizing a goodness-of-fit score~\citep{BIC} and have been extended through continuous optimization to handle larger-scale graphs~\citep{NOTEARS};
and functional causal models (FCMs)~\citep{ICALiNGAM}, which exploit specific data-generating assumptions to achieve stronger identifiability~\citep{QPE}.
Nevertheless, causal discovery from complex observational data remains inherently challenging, motivating the exploration of perspectives beyond these established methodologies.

In this paper, we propose Decoupled Causal Discovery (DCD), a novel decoupling-based perspective that does not rely on CITs, structure scores, or restrictive functional assumptions.
The core mechanism of DCD is to decouple the distribution via weighting functions that render non-target variables mutually independent.
We theoretically show that under this decoupled distribution, a variable preserves dependence on the target if and only if it belongs to the MB of the target, thereby enabling the direct identification of MBs.
Building upon this identifiability, DCD iteratively constructs the Completed Partially Directed Acyclic Graph (CPDAG) by exploiting the structural asymmetries within the MBs.

The main contributions of this paper are summarized as follows:

\begin{itemize}[topsep=0pt, left=5pt]

\item We provide a decoupling-based perspective for causal discovery and propose DCD, a novel framework that is orthogonal to CITs, structure scores, and FCM assumptions.

\item We establish the theoretical foundation for this framework, proving the identifiability of MBs under the decoupled distribution, and the soundness and completeness of the construction algorithm.

\item Empirical evaluations show that DCD achieves strong performance, particularly under challenging noise regimes where several representative baselines become unreliable.
\end{itemize}

\section{Related Work}\label{sec:related_work}

We outline several representative and recent causal discovery methods across three classical paradigms, while referring to~\citet{nogueira2022methods} and~\citet{vowels2022d} for a comprehensive overview.
As a modeling paradigm with strong identifiability, FCMs exploit asymmetries in structural data-generating assumptions to identify causal order.
Pioneered by linear non-Gaussian models~\citep{ICALiNGAM, DirectLiNGAM, LiNGAM_L}, research has broadened the scope of such identifiability~\citep{ANM, CAM, PNL, HNM}, with recent work moving toward a more unified framework~\citep{CVM, QPE, LiNG_LC}.
Score-based methods maximize a goodness-of-fit score~\citep{BIC, BDe, GeneralizedScore} over the space of graph structures.
Classical approaches perform direct search over structures~\citep{K2, GES}, while continuous optimization \citep{NOTEARS, GraNDAG, GOLEM} has enabled this paradigm to scale to larger graphs.

Constraint-based methods~\citep{PC, FCI}, representing one of the earliest paradigms~\citep{SGS, IC}, rely on CITs~\citep{PCStable, CITsurvey} to infer graphical separation (e.g., d-separation~\citep{pearl1988}) in a causal graph~\citep{SGS2, peters2017elements}.
To perform CITs more efficiently, a natural strategy is to exploit local structure.
Some early approaches first identified MBs and then assembled the global structure~\citep{GS1999, MB2008}.
More recently, a line of research~\citep{recu2021b, recu2021a, recu2022, recu2023, recu2025} has further exploited MBs to locally identify removable variables, thereby progressively reducing the problem size during algorithm execution.
Despite leveraging local structural information from MBs, these methods are ultimately driven by CITs.

\section{Decoupled Causal Discovery}

We first introduce the procedure for MB identification (Section~\ref{subsec:mb}), followed by the construction of the global CPDAG by directly identifying MBs (Section~\ref{subsec:global}), and present the practical implementation (Section~\ref{subsec:practical_e}). 
For notational readability, we slightly abuse the notation by treating sets of random variables and random vectors interchangeably.
Thus, calligraphic and boldface symbols may be subjected to both set-theoretic and probabilistic operations.

\subsection{Identifiability of MB via Decoupling} \label{subsec:mb}

In this section, we elucidate the direct identification of the MB via distribution decoupling.
Notably, this process operates independently of the graph structure, relying purely on statistical relationships.
This implies that the proposed procedure serves as a general statistical identification framework with broad applicability in scenarios beyond causal discovery.
We first provide the formal definition of the Markov blanket and Markov boundary~\citep{pearl1988}.

\begin{definition}[Markov Blanket and Markov Boundary] \label{def:mb}
    Given a set of random variables $\mathcal{O}$ and a target variable $X_t \in \mathcal{O}$, a \textbf{Markov blanket} of $X_t$ relative to $\mathcal{O}$ is a subset $\mathbf{M} \subseteq \mathcal{O} \setminus\{X_t\}$ such that
    \begin{equation*}
        X_t \perp\!\!\!\perp \bigl( \mathcal{O} \setminus (\{X_t\} \cup \mathbf{M}) \bigr) \mid \mathbf{M}.
    \end{equation*}
    The \textbf{Markov boundary (MB)} of $X_t$ relative to $\mathcal{O}$, denoted by $\MB{X_t \mid \mathcal{O}}$, is defined as an inclusion-minimal Markov blanket. Formally, it is a Markov blanket $\mathbf{M}$ that satisfies:
    \begin{equation*}
        \forall \mathbf{M}^{\prime} \subsetneq \mathbf{M}, \quad X_t \not\!\perp\!\!\!\perp \bigl( \mathcal{O} \setminus (\{X_t\} \cup \mathbf{M}^{\prime}) \bigr) \mid \mathbf{M^{\prime}}.
    \end{equation*}
\end{definition}

While prior works occasionally conflate the Markov blanket with the Markov boundary (the minimal conditioning set), we explicitly distinguish between the two to ensure theoretical rigor.

We next introduce the core principle of DCD, which leverages a weighting scheme to decouple the joint distribution of non-target variables, thereby enabling direct identification of the MB for a target variable. 
This framework draws inspiration from the line of work on stable prediction~\citep{kuang2018stable, kuang2023stable, shen2018, shen2020, xu2022, cui2022stable}, and our theoretical results can be viewed as a generalization of that paradigm.
We begin with the definitions of weighting and distribution decoupling.

\begin{definition}[Weighting Function and Weighted Distribution] \label{def:w}

Given a target variable $X_t$ and a random vector $\mathbf{Z}$ with domain $\text{Val}(\mathbf{Z})$, the weighting function set $\mathcal{W}(\mathbf{Z})$ is defined as:
\begin{equation*}
    \mathcal{W}(\mathbf{Z}) = \{w: \text{Val}(\mathbf{Z}) \rightarrow \mathbb{R}^+ \mid \mathbb{E}[w(\mathbf{Z})] = 1\}.
\end{equation*}
For any $w \in \mathcal{W}(\mathbf{Z})$, the weighted distribution $\tilde{P}_w(X_t, \mathbf{Z})$ is defined as:
\begin{equation*}
    \tilde{P}_w(X_t, \mathbf{Z}) = w(\mathbf{Z}) P(X_t, \mathbf{Z}).
\end{equation*}
\end{definition}

Intuitively, Definition~\ref{def:w} characterizes the space of valid weighting functions that alter the distribution of non-target variables.
Unlike~\citet{xu2022}  (for linear models), our definition does not impose first-moment balance conditions like $\mathbb{E}[w(\mathbf{Z})\mathbf{Z}] = \mathbf{0}$. We then define the subset of weights that eliminates the statistical entanglement among $\mathbf{Z}$.

\begin{definition}[Decoupled Distribution] \label{def:decoupled}
Given the target variable $X_t$ and the random vector $\mathbf{Z}$, the set of decoupling weights is defined as:
\begin{equation}
    \mathcal{W}_{\perp}(\mathbf{Z}) = \{w \in \mathcal{W}(\mathbf{Z}) \mid \mathbf{Z} \text{ are mutually independent in } \tilde{P}_w\}.
    \label{eq:Wperp}
\end{equation}
For any $w \in \mathcal{W}_{\perp}(\mathbf{Z})$, we refer to $\tilde{P}_w(X_t, \mathbf{Z})$ as a decoupled distribution of $\mathbf{Z}$.
\end{definition}

In a decoupled distribution $\tilde{P}_w$, the joint distribution factorizes as $\tilde{P}_w(\mathbf{Z}) = \prod_{i} \tilde{P}_w(Z_i)$, meaning that $\tilde{P}_w(\mathbf{Z}_j \mid \mathbf{Z}_k) = \tilde{P}_w(\mathbf{Z}_j)$ for any disjoint $\mathbf{Z}_j, \mathbf{Z}_k$, which is essential for isolating non-MB variables. In general, the set of such decoupling weights $\mathcal{W}_{\perp}(\mathbf{Z})$ is non-trivial (see Appendix~\ref{app:future}).

We next introduce two assumptions necessary for MB identifiability under the decoupled distribution.

\begin{assumption}[Strict Positivity]
\label{ass:pos}
The joint distribution (and consequently all marginal and conditional distributions) of random variables is strictly positive.
\end{assumption}

Assumption~\ref{ass:pos} is a standard regularity condition widely adopted in stable prediction and probabilistic graphical models~\citep{xu2022, pearl2009causality}.
Combined with the strictly positive weighting functions defined in Definition~\ref{def:w}, Assumption~\ref{ass:pos} ensures that all probability ratios and conditioning operations in our subsequent analysis are well-defined.
Crucially, this assumption guarantees that the conditional independence relation satisfies the \textit{intersection axiom}~\citep{pearl1988}, which ensures the uniqueness of the MB (Appendix~\ref{app:proofs}, Lemmas~\ref{lemma:intersection} and~\ref{lemma:mb_uniqueness}).

While the strict positivity guarantees the fundamental properties, we further introduce a non-cancellation assumption to rule out pathological cases where the influence of a relevant variable vanishes merely due to numerical coincidences.

\begin{assumption}[Weighting Faithfulness]
\label{ass:wf}
Let $\mathcal{O}$ be a set of variables, $X_t \in \mathcal{O}$ be the target variable, and $\mathbf{X}_{-t} = \mathcal{O} \setminus \{X_t\}$.
Given a decoupling weight $w \in \mathcal{W}_{\perp}(\mathbf{X}_{-t})$ (and its corresponding $\tilde{P}_w(X_t, \mathbf{X}_{-t})$), we assume that for any $X_e \in \mathbf{X}_{-t}$ and $\mathbf{X}_c \subseteq \mathbf{X}_{-t} \setminus \{X_e\}$:
\begin{equation*}
\text{If } P(X_t \mid X_e, \mathbf{X}_c) \text{ depends on } X_e, \text{ then so does } \mathbb{E}_{\tilde{P}_w(\mathbf{X}_c)}[P(X_t \mid X_e, \mathbf{X}_c)].
\end{equation*}

\end{assumption}

\begin{remark}
This assumption states that integrating out $\mathbf{X}_c$ under $\tilde{P}_w$ does not perfectly cancel the original dependence of $X_t$ on $X_e$ given $\mathbf{X}_c$. 
This generally holds when the expectation is taken over the original distribution.
Furthermore, since the weight $w \in \mathcal{W}_{\perp}(\mathbf{X}_{-t})$ is derived solely from $P(\mathbf{X}_{-t})$, it is strictly agnostic to the mechanism $P(X_t \mid \cdot)$.
Consequently, it is extremely unlikely for the decoupled distribution $\tilde{P}_w$ to adversarially align with $P(X_t \mid \cdot)$ to induce an exact cancellation.
We provide a formal proof showing that this assumption holds almost everywhere (Proposition~\ref{pro:gwf}).
\end{remark}

\begin{remark}
This assumption mirrors the standard faithfulness (Assumption~\ref{ass:3}, although our current MB discussion is purely statistical), serving as an intuitive and mild condition in practice.
Both assumptions ensure that true statistical dependencies are not masked by exact numerical cancellations, thereby guaranteeing that the underlying relationships can be properly identified.
\end{remark}

\begin{remark}
This assumption parallels several standard concepts in causal inference:
\begin{itemize}[topsep=0pt, partopsep=0pt, itemsep=0pt, left=2pt,]
\item \textit{Covariate Balancing:}
The term $\mathbb{E}_{\tilde{P}_w(\mathbf{X}_c)}[P(X_t \mid X_e, \mathbf{X}_c)]$ matches the standard form of estimators using covariate balancing~\citep{rosenbaum1983central}.
Assuming it is non-degenerate is a standard implicit setting for causal inference, ensuring that heterogeneous effects do not perfectly cancel out, which is necessary for identification.
\item \textit{Backdoor Adjustment:}
If the expectation is taken over the original distribution, it recovers the formula of backdoor adjustment~\citep{pearl1995causal, pearl2009causality}. In this context, this condition acts as a natural guarantee that the adjusted effect does not coincidentally vanish if a genuine effect exists.
\end{itemize}
\end{remark}

Based on the definitions and the two assumptions above, we establish the following property to achieve the identifiability of the MB (see Appendix~\ref{app:mb_proof} for the complete proof).

\begin{proposition}[MB Conditional Distribution Invariance]
\label{pro:inv}
Let $\mathcal{O}$ be a set of variables, $X_t \in \mathcal{O}$ be the target variable, and $\mathbf{X}_{-t} = \mathcal{O} \setminus \{X_t\}$. Under Assumption~\ref{ass:pos}, for any weighting function $w \in \mathcal{W}(\mathbf{X}_{-t})$ (and its corresponding $\tilde{P}_w(X_t, \mathbf{X}_{-t})$), the following holds for any $\mathbf{X}_{\textup{MB}}^{+}$ satisfying $\MB{X_t \mid \mathcal{O}} \subseteq \mathbf{X}_{\textup{MB}}^{+} \subseteq \mathbf{X}_{-t}$:
\begin{equation*}
\tilde{P}_w(X_t \mid \mathbf{X}_{\textup{MB}}^+)
=
P(X_t \mid \mathbf{X}_{\textup{MB}}^+)
=
P(X_t \mid \MB{X_t \mid \mathcal{O}})
=
P(X_t \mid \mathbf{X}_{-t})
.
\end{equation*}

\end{proposition}

Note that this proposition holds for any weighted distribution (Definition~\ref{def:w}, not necessarily decoupled), requiring only the strict positivity assumption.
It extends the conditional distribution invariance of unweighted variables given all weighted variables \citep{shimodaira2000improving, xu2022} to the context of MBs.
Specifically, as long as the conditioning set of the target variable in the weighted distribution contains the complete MB, the resulting conditional distribution remains identical to that under the original distribution, both coinciding with $P(X_t \mid \mathbf{X}_{-t})$ and $P(X_t \mid \MB{X_t \mid \mathcal{O}})$.
This also implies that the original MB remains a Markov blanket under weighted distributions.

Building on this proposition, we further establish the main theorem regarding MB identifiability.

\begin{theorem}[MB Identifiability] \label{the:MB}
Let $\mathcal{O}$ be a set of variables, $X_t \in \mathcal{O}$ be the target variable, and $\mathbf{X}_{-t} = \mathcal{O} \setminus \{X_t\}$.
Under Assumption~\ref{ass:pos}, for any decoupling weight $w \in \mathcal{W}_{\perp}(\mathbf{X}_{-t})$ (and its corresponding $\tilde{P}_w(X_t, \mathbf{X}_{-t})$) satisfying Assumption~\ref{ass:wf}, we have for any variable $X_i \in \mathbf{X}_{-t}$:
\begin{equation*}
    X_i \in \MB{X_t \mid \mathcal{O}} \iff X_i \not\!\perp\!\!\!\perp_{\tilde{P}_w}\!\! X_t.
\end{equation*}
\end{theorem}

According to Theorem~\ref{the:MB}, MB identification can be a direct process, achieved by learning decoupling weights over the non-target variables and evaluating which variables remain dependent on the target variable under the weighted distribution.

Importantly, our conclusion relies exclusively on the statistical definition of the MB along with two basic assumptions, without requiring any specific model specifications or graph structures.
In this context, our result can be viewed as a generalization of the findings in \citet{xu2022}.
Their theoretical objective focuses on the mean prediction problem and therefore concerns the \textit{minimal stable variable set}, which is the subset of the MB that has an effect on the first moment.
Additionally, their analysis is grounded in an additive noise model and a linear estimation framework.
In contrast, beyond the standard strict positivity, our framework requires only a single non-coincidental assumption.
Furthermore, we show that Assumption~\ref{ass:wf} holds almost everywhere under a natural measure on the weight space, confirming its genericity.

\begin{condition}[Non-degeneracy of Decoupling Space]
\label{ass:nd}
The space of decoupled distributions of $\mathbf{Z}$ (generated by $w \in \mathcal{W}_{\perp}(\mathbf{Z})$) is sufficiently rich to avoid degenerating onto any hyperplane.

Formally, let $P_{\mathbf{Z}}$ be the original density of $\mathbf{Z}$, and equip $\mathcal{W}_{\perp}(\mathbf{Z})$ with a Borel probability measure $\mu_{\mathcal{W}}^{\mathbf{Z}}$. Then, for any nonzero continuous linear functional $\ell: L^1(\lambda_{\mathbf{Z}}) \rightarrow \mathbb{R}$ (where $\lambda_\mathbf{Z}$ is the dominating measure on $\mathbf{Z}$), the following holds:
\begin{equation*}
\mu_{\mathcal{W}}^{\mathbf{Z}} \bigl( \bigl\{ w \in \mathcal{W}_{\perp}(\mathbf{Z}) \;\big\vert{}\; \ell(w \cdot P_{\mathbf{Z}}) = 0 \bigr\} \bigr) = 0
\end{equation*}

\end{condition}

Condition~\ref{ass:nd} formalizes that the functional space of valid weights is sufficiently rich (discussed in Appendix~\ref{app:future}) and characterizes a practical estimation process through the probability measure $\mu_{\mathcal{W}}^{\mathbf{Z}}$.
Under this condition, we establish that Assumption~\ref{ass:wf} holds almost everywhere with respect to $\mu_{\mathcal{W}}^{\mathbf{Z}}$.

\begin{proposition}[Genericity of Weighting Faithfulness]
\label{pro:gwf}
Let $\mathcal{O}$ be a set of variables, $X_t \in \mathcal{O}$ be the target variable, and $\mathbf{X}_{-t} = \mathcal{O} \setminus \{X_t\}$.
Under Assumption~\ref{ass:pos} and Condition~\ref{ass:nd}, for $\mu_{\mathcal{W}}^{\mathbf{X}_{-t}}$-almost every weight $w \in \mathcal{W}_{\perp}(\mathbf{X}_{-t})$, Assumption~\ref{ass:wf} (Weighting Faithfulness) holds.
\end{proposition}

Consequently, Proposition~\ref{pro:gwf} confirms that Assumption~\ref{ass:wf} is a generic property in practical scenarios, providing a theoretical foundation for our framework.

\subsection{Global Graph Construction} \label{subsec:global}

\begin{algorithm}[t]
\small
\linespread{0.95}\selectfont
\caption{\small DCD-C: Construct CPDAG from Local MBs}
\label{alg:dcd-c}
\begin{algorithmic}[1] 
\REQUIRE Set of variables $\mathcal{V} = \{X_1, \dots, X_d\}$, Markov boundary identifier $\textup{MB}(\cdot \mid \cdot)$
\ENSURE Completed Partially Directed Acyclic Graph (CPDAG) $\hat{\mathcal{G}} = (\mathcal{V}, \hat{\mathcal{E}})$
\STATE Initialize graph $\hat{\mathcal{G}} = (\mathcal{V}, \hat{\mathcal{E}})$ with empty edge set $\hat{\mathcal{E}} \leftarrow \emptyset$
\STATE Initialize v-structure set $\mathcal{K} \leftarrow \emptyset$
\FOR{each node $X_t \in \mathcal{V}$}
    \STATE Obtain initial Markov boundary: $\mathcal{M}_t \leftarrow  \textup{MB}(X_t \mid \mathcal{V})$
    \STATE Initialize neighbor set $\mathcal{N}_t \leftarrow \mathcal{M}_t$
    \STATE \COMMENT{Traverse by increasing subset size to find exclusion sets}
    \FOR{$k = 1$ \TO $|\mathcal{M}_t| - 1$}
        \FOR{each subset $\mathcal{Z} \subset \mathcal{M}_t$ with $|\mathcal{Z}| = k$}
            \STATE $\mathcal{M}_{\setminus \mathcal{Z}} \leftarrow  \textup{MB}(X_t \mid (\mathcal{M}_t \cup \{X_t\}) \setminus \mathcal{Z})$
            \COMMENT{Local MB within the restricted subset}
            \STATE $\mathcal{D}_{\text{drop}} \leftarrow (\mathcal{M}_t \setminus \mathcal{Z}) \setminus \mathcal{M}_{\setminus \mathcal{Z}}$
            \IF{$\mathcal{D}_{\text{drop}} \neq \emptyset$} 
                \FOR{each $Y \in \mathcal{D}_{\text{drop}}$} 
                    \FOR{each $Z \in \mathcal{Z}$}
                        \STATE $\mathcal{K} \leftarrow \mathcal{K} \cup \{(X_t, Z, Y)\}$
                        \COMMENT{Identify potential v-structures}
                    \ENDFOR
                    \STATE $\mathcal{N}_t \leftarrow \mathcal{N}_t \setminus \{Y\}$ 
                    \COMMENT{Remove spouse from candidate neighbors}
                \ENDFOR
            \ENDIF
        \ENDFOR
    \ENDFOR
    \ \COMMENT{Remaining nodes in $\mathcal{N}_t$ are neighbors of $X_t$, form the undirected skeleton}
    \FOR{each $W \in \mathcal{N}_t$}
        \STATE Add undirected edge to $\hat{\mathcal{E}}$: $X_t - W$
    \ENDFOR
\ENDFOR
\FOR{each $\{(X, Z, Y)\} \in \mathcal{K}$}
    \IF{$Z$ is adjacent to both $X$ and $Y$ in $\hat{\mathcal{G}}$}
    \IF{$X$ and $Y$ are not adjacent in $\hat{\mathcal{G}}$}
        \STATE Orient edges in $\hat{\mathcal{E}}$ as: $X \rightarrow Z$ and $Y \rightarrow Z$
    \ENDIF
    \ENDIF
\ENDFOR
\STATE Apply Meek rules~\citep{Meek} to $\hat{\mathcal{G}}$ to exhaustively orient remaining undirected edges
\RETURN $\hat{\mathcal{G}}$
\end{algorithmic}
\end{algorithm}

Let $\mathbf{X} = (X_1, \dots, X_d)$ be a random vector following a joint probability distribution $P$. Let $\mathcal{G} = (\mathcal{V}, \mathcal{E})$ be a directed acyclic graph (DAG) characterizing the underlying causal structure of $\mathbf{X}$, where the node set $\mathcal{V} = \{X_1, \dots, X_d\}$ corresponds to the random variables. A directed edge from $X_i$ to $X_j$, denoted by $(X_i, X_j) \in \mathcal{E}$, indicates that $X_i$ is a direct cause of $X_j$. For any node $X_i \in \mathcal{V}$, we denote its parents, children, and spouses in $\mathcal{G}$ as $\pa{X_i}$, $\ch{X_i}$, and $\spn{X_i}$, respectively. Specifically, $\pa{X_i} = \{X_p \mid (X_p, X_i) \in \mathcal{E}\}$, $\ch{X_i} = \{X_c \mid (X_i, X_c) \in \mathcal{E}\}$ and $\spn{X_i} = \{X_s \in \mathcal{V} \setminus \{X_i\} \mid \ch{X_i} \cap \ch{X_s} \neq \emptyset\}$. We adopt standard assumptions~\citep{SGS2} to bridge the distribution $P$ with the graph $\mathcal{G}$.

\begin{assumption}[Causal Sufficiency] \label{ass:1}
There are no unobserved confounders. That is, for any $X_i, X_j \in \mathcal{V}$, if a variable $X_u$ is a common cause of $X_i$ and $X_j$, then $X_u \in \mathcal{V}$.
\end{assumption}

\begin{assumption}[Markov Condition] \label{ass:2}
    The joint distribution $P$ satisfies the Markov condition with respect to $\mathcal{G}$. That is, every variable $X_i \in \mathcal{V}$ is conditionally independent of its non-descendants given its parents $\pa{X_i}$.
\end{assumption}

\begin{assumption}[Faithfulness] \label{ass:3}
    The joint distribution $P$ is faithful to the DAG $\mathcal{G}$. This implies that all conditional independencies present in $P$ are exactly those entailed by the d-separation in $\mathcal{G}$. Formally, for any disjoint subsets $\mathbf{U}, \mathbf{V}, \mathbf{W} \subset \mathcal{V}$:
    \begin{equation*}
      \mathbf{U} \perp\!\!\!\perp \mathbf{V} \mid \mathbf{W} \text{ in } P \iff \mathbf{U} \text{ and } \mathbf{V} \text{ are d-separated by } \mathbf{W} \text{ in } \mathcal{G}.
    \end{equation*}
\end{assumption}

Under the three assumptions above, the MB of any variable in $\mathcal{G}$ consists exactly of its parents, children, and spouses (Lemma~\ref{lemma:mb}).
Building upon this property, we propose DCD-C (Algorithm~\ref{alg:dcd-c}), an algorithm designed to recover the CPDAG solely via MB identification.

The core idea of DCD-C relies on the structural asymmetry between direct neighbors (parents and children) and spouses.
If a variable is a true neighbor of a target, it will always remain in the target's MB across any subset of variables containing it.
Conversely, this is not necessarily true for a spouse.

DCD-C (Algorithm~\ref{alg:dcd-c}) exploits this property to isolate true neighbors.
For each variable, the algorithm first identifies its MB over the full set of variables to form an initial set of candidate neighbors.
We then eliminate non-neighbors by iteratively examining the MBs of the target variable over subsets of the full MB.
Specifically, in each iteration, we exclude a chosen subset from the candidate neighbors and re-evaluate the MB over the remaining variables.
If any additional variable drops out of the MB, it is guaranteed to be a non-adjacent node.
We establish theoretical guarantees for DCD-C (see Appendix~\ref{app:ag_proof} for the complete proof).
\begin{theorem}[Soundness and Completeness of DCD-C] \label{the:AG}
Under Assumptions~\ref{ass:1}, \ref{ass:2}, and~\ref{ass:3}, with an oracle MB identifier $\textup{MB}(\cdot \mid \cdot)$, Algorithm~\ref{alg:dcd-c} (DCD-C) returns the true CPDAG of $\mathcal{G}$.
\end{theorem}
This theorem guarantees the strict correctness of DCD-C in recovering the structure up to its Markov equivalence class.
While DCD-C requires iterating over subsets of the MB, its time complexity is bounded by $O(d \cdot 2^{m})$ MB identifications, where $m$ denotes the maximum MB size.
Consequently, the algorithm is highly efficient for sparse graphs where $m$ remains small.
Importantly, since the operations are localized to individual variables, DCD-C inherently offers potential for parallel execution.
We provide a detailed discussion on the computational complexity in Appendix~\ref{app:complexity}.

\subsection{Practical Estimation} \label{subsec:practical_e}

Building upon the theoretical identifiability, we propose a practical framework for estimating MBs.
We learn the sample weights as an optimization problem via backpropagation on a decoupling metric, and subsequently perform weighted independence tests to infer relations under the decoupled distribution.
Equipped with this, the overall discovery process follows Algorithm~\ref{alg:dcd-c} to construct the global graph.

\textit{\textbf{Weight Parameterization.}}
To ensure the sample weights are non-negative and satisfy the expectation constraints in Definition~\ref{def:w}, we parameterize them using a softmax function.
Given $n$ samples and learnable parameters $\theta \in \mathbb{R}^n$, the weight for the $k$-th sample is defined as $w_k = n \cdot \exp(\theta_k) / \sum_{l=1}^n \exp(\theta_l)$.
This parameterization guarantees that $\sum_{k=1}^n w_k = n$, thereby empirically satisfying $\mathbb{E}[w] = 1$.

\textit{\textbf{Decoupling via Weighted HSIC.}}
According to the definition, the decoupled distribution requires mutual independence among the non-target variables $\mathbf{X}_{-t}$.
To empirically enforce this, we minimize the pairwise weighted Hilbert-Schmidt Independence Criterion (HSIC) using a Radial Basis Function (RBF) kernel with median-heuristic bandwidth.
The empirical objective function is formulated as:
\begin{equation*}
    \mathcal{L}(\theta) = \sum_{i \neq j, X_i, X_j \in \mathbf{X}_{-t}} \widehat{\text{HSIC}}_w(X_i, X_j)
\end{equation*}
where $\widehat{\text{HSIC}}_w$ denotes the empirical weighted HSIC estimator. To prevent weight degeneration during optimization, we monitor Kish's Effective Sample Size, defined as $\text{ESS} = (\sum_{k=1}^n w_k)^2 / \sum_{k=1}^n w_k^2$, and employ early stopping if it drops below a predefined threshold.

\textit{\textbf{Weighted Correlation Test.}}
Once the optimal weights $w^*$ are learned via gradient descent, we proceed to the identification step dictated by Theorem~\ref{the:MB}.
For each non-target variable $X_i \in \mathbf{X}_{-t}$, we compute its weighted Pearson correlation coefficient $r_w$ with the target variable $X_t$.
We then conduct a two-tailed Student's t-test with degrees of freedom adjusted by the ESS, using a default significance level of $\alpha=0.05$.
Ultimately, we include $X_i$ in the estimated MB if the null hypothesis is rejected.

\section{Experiments}\label{sec:exp}

We empirically evaluate the effectiveness of our DCD under different mechanisms (Section~\ref{sec:expM}), noise conditions (Section~\ref{sec:expN}), graph topologies (Section~\ref{sec:expG}), and sample sizes (Section~\ref{sec:expS}), comparing against representative baselines: 
PC~\citep{PC}, GES~\citep{GES}, ICALiNGAM~\citep{ICALiNGAM}, DirectLiNGAM~\citep{DirectLiNGAM}, Notears~\citep{NOTEARS}, GraNDAG~\citep{GraNDAG}, and GOLEM~\citep{GOLEM}. Notably, for the CIT in the PC algorithm, we implement the Kernel-based CIT (KCIT)~\citep{KCIT}, the Fisher-z test (FisherZ)~\citep{FisherZ}, and their respective ensemble variants (E-KCIT/E-FisherZ)~\citep{ECIT}, to achieve optimal empirical performance.
Since the evaluated methods output either DAG or CPDAG, we uniformly convert all estimated structures into their corresponding CPDAG to ensure a fair and consistent comparison. We evaluate the performance using standard metrics: Structural Hamming Distance adapted for CPDAGs ($\text{SHD}_c$), Precision, Recall, and F1-score.

\subsection{Performance Across Mechanisms}\label{sec:expM}

\begin{table*}[t]
\centering
\caption{Performance under different mechanisms across graph sizes $d \in \{5, 10, 20\}$.}
\label{tab:main}
\small
\setlength{\tabcolsep}{3.5pt}
\renewcommand{\arraystretch}{0.95}
\begin{tabular}{l|cccc|cccc}
\toprule
\multicolumn{1}{c|}{\multirow{2.75}{*}{Method}}
& \multicolumn{4}{c|}{\textbf{Func}}
& \multicolumn{4}{c}{\textbf{MLP}} \\
\cmidrule(lr){2-5} \cmidrule(lr){6-9}
& SHD$_c \downarrow$ & Precision $\uparrow$ & Recall $\uparrow$ & F1 $\uparrow$
& SHD$_c \downarrow$ & Precision $\uparrow$ & Recall $\uparrow$ & F1 $\uparrow$ \\
\midrule
\multicolumn{9}{c}{$d=5$} \\
\midrule

DCD
& \textbf{2.4}$_{\text{$\pm$\textbf{1.1}}}$
& \textbf{0.86}$_{\text{$\pm$\textbf{0.31}}}$
& \textbf{0.43}$_{\text{$\pm$\textbf{0.24}}}$
& \textbf{0.55}$_{\text{$\pm$\textbf{0.26}}}$
& \textbf{0.9}$_{\text{$\pm$\textbf{1.1}}}$
& \textbf{0.97}$_{\text{$\pm$\textbf{0.15}}}$
& \textbf{0.79}$_{\text{$\pm$\textbf{0.26}}}$
& \textbf{0.85}$_{\text{$\pm$\textbf{0.22}}}$
\\

PC (FisherZ)
& 3.5$_{\text{$\pm$1.2}}$
& 0.41$_{\text{$\pm$0.34}}$
& 0.30$_{\text{$\pm$0.25}}$
& 0.34$_{\text{$\pm$0.26}}$
& 3.7$_{\text{$\pm$1.7}}$
& 0.32$_{\text{$\pm$0.36}}$
& 0.28$_{\text{$\pm$0.29}}$
& 0.29$_{\text{$\pm$0.30}}$
\\

PC (E-FisherZ)
& 3.4$_{\text{$\pm$1.3}}$
& 0.41$_{\text{$\pm$0.38}}$
& 0.28$_{\text{$\pm$0.24}}$
& 0.32$_{\text{$\pm$0.28}}$
& 3.4$_{\text{$\pm$1.5}}$
& 0.32$_{\text{$\pm$0.36}}$
& 0.28$_{\text{$\pm$0.30}}$
& 0.29$_{\text{$\pm$0.31}}$
\\

PC (KCIT)
& 3.2$_{\text{$\pm$1.2}}$
& 0.38$_{\text{$\pm$0.41}}$
& 0.26$_{\text{$\pm$0.28}}$
& 0.30$_{\text{$\pm$0.31}}$
& 2.1$_{\text{$\pm$1.7}}$
& 0.52$_{\text{$\pm$0.42}}$
& 0.52$_{\text{$\pm$0.42}}$
& 0.52$_{\text{$\pm$0.41}}$
\\

PC (E-KCIT)
& 3.2$_{\text{$\pm$1.1}}$
& 0.37$_{\text{$\pm$0.41}}$
& 0.25$_{\text{$\pm$0.27}}$
& 0.28$_{\text{$\pm$0.30}}$
& 2.7$_{\text{$\pm$1.6}}$
& 0.39$_{\text{$\pm$0.40}}$
& 0.36$_{\text{$\pm$0.37}}$
& 0.37$_{\text{$\pm$0.38}}$
\\

GES
& 3.1$_{\text{$\pm$1.6}}$
& 0.53$_{\text{$\pm$0.41}}$
& 0.38$_{\text{$\pm$0.30}}$
& 0.43$_{\text{$\pm$0.33}}$
& 2.4$_{\text{$\pm$1.7}}$
& 0.60$_{\text{$\pm$0.39}}$
& 0.54$_{\text{$\pm$0.35}}$
& 0.56$_{\text{$\pm$0.35}}$
\\

ICALiNGAM
& 3.0$_{\text{$\pm$0.9}}$
& 0.68$_{\text{$\pm$0.47}}$
& 0.26$_{\text{$\pm$0.21}}$
& 0.37$_{\text{$\pm$0.28}}$
& 2.9$_{\text{$\pm$1.0}}$
& 0.66$_{\text{$\pm$0.48}}$
& 0.27$_{\text{$\pm$0.25}}$
& 0.36$_{\text{$\pm$0.30}}$
\\

DirectLiNGAM
& 2.9$_{\text{$\pm$0.9}}$
& 0.70$_{\text{$\pm$0.46}}$
& 0.28$_{\text{$\pm$0.22}}$
& 0.38$_{\text{$\pm$0.28}}$
& 2.9$_{\text{$\pm$1.0}}$
& 0.66$_{\text{$\pm$0.48}}$
& 0.27$_{\text{$\pm$0.26}}$
& 0.36$_{\text{$\pm$0.31}}$
\\

Notears
& 3.3$_{\text{$\pm$0.8}}$
& 0.52$_{\text{$\pm$0.51}}$
& 0.18$_{\text{$\pm$0.20}}$
& 0.25$_{\text{$\pm$0.27}}$
& 3.6$_{\text{$\pm$0.6}}$
& 0.30$_{\text{$\pm$0.46}}$
& 0.10$_{\text{$\pm$0.16}}$
& 0.14$_{\text{$\pm$0.23}}$
\\

GraNDAG
& 3.9$_{\text{$\pm$2.0}}$
& 0.45$_{\text{$\pm$0.43}}$
& 0.32$_{\text{$\pm$0.28}}$
& 0.34$_{\text{$\pm$0.31}}$
& 3.2$_{\text{$\pm$2.1}}$
& 0.55$_{\text{$\pm$0.38}}$
& 0.50$_{\text{$\pm$0.33}}$
& 0.49$_{\text{$\pm$0.32}}$
\\

GOLEM
& 3.0$_{\text{$\pm$0.9}}$
& 0.64$_{\text{$\pm$0.48}}$
& 0.25$_{\text{$\pm$0.22}}$
& 0.35$_{\text{$\pm$0.29}}$
& 3.2$_{\text{$\pm$0.8}}$
& 0.60$_{\text{$\pm$0.49}}$
& 0.21$_{\text{$\pm$0.21}}$
& 0.30$_{\text{$\pm$0.28}}$
\\

\midrule
\multicolumn{9}{c}{$d=10$} \\
\midrule

DCD
& \textbf{5.5}$_{\text{$\pm$\textbf{1.7}}}$
& 0.87$_{\text{$\pm$0.16}}$
& \textbf{0.45}$_{\text{$\pm$\textbf{0.16}}}$
& \textbf{0.58}$_{\text{$\pm$\textbf{0.16}}}$
& \textbf{3.3}$_{\text{$\pm$\textbf{2.0}}}$
& 0.88$_{\text{$\pm$0.16}}$
& \textbf{0.72}$_{\text{$\pm$\textbf{0.18}}}$
& \textbf{0.78}$_{\text{$\pm$\textbf{0.17}}}$
\\

PC (FisherZ)
& 9.4$_{\text{$\pm$2.9}}$
& 0.28$_{\text{$\pm$0.22}}$
& 0.24$_{\text{$\pm$0.17}}$
& 0.26$_{\text{$\pm$0.19}}$
& 10.4$_{\text{$\pm$2.5}}$
& 0.15$_{\text{$\pm$0.15}}$
& 0.17$_{\text{$\pm$0.17}}$
& 0.16$_{\text{$\pm$0.16}}$
\\

PC (E-FisherZ)
& 8.6$_{\text{$\pm$2.4}}$
& 0.29$_{\text{$\pm$0.25}}$
& 0.22$_{\text{$\pm$0.18}}$
& 0.25$_{\text{$\pm$0.20}}$
& 9.1$_{\text{$\pm$2.6}}$
& 0.20$_{\text{$\pm$0.19}}$
& 0.21$_{\text{$\pm$0.19}}$
& 0.20$_{\text{$\pm$0.19}}$
\\

PC (KCIT)
& 8.2$_{\text{$\pm$1.5}}$
& 0.19$_{\text{$\pm$0.18}}$
& 0.15$_{\text{$\pm$0.13}}$
& 0.16$_{\text{$\pm$0.14}}$
& 6.3$_{\text{$\pm$2.9}}$
& 0.37$_{\text{$\pm$0.30}}$
& 0.37$_{\text{$\pm$0.30}}$
& 0.37$_{\text{$\pm$0.30}}$
\\

PC (E-KCIT)
& 7.9$_{\text{$\pm$1.3}}$
& 0.21$_{\text{$\pm$0.22}}$
& 0.14$_{\text{$\pm$0.14}}$
& 0.16$_{\text{$\pm$0.15}}$
& 6.4$_{\text{$\pm$2.3}}$
& 0.33$_{\text{$\pm$0.27}}$
& 0.32$_{\text{$\pm$0.25}}$
& 0.32$_{\text{$\pm$0.26}}$
\\

GES
& 8.5$_{\text{$\pm$2.7}}$
& 0.34$_{\text{$\pm$0.29}}$
& 0.25$_{\text{$\pm$0.20}}$
& 0.29$_{\text{$\pm$0.23}}$
& 9.3$_{\text{$\pm$3.5}}$
& 0.21$_{\text{$\pm$0.28}}$
& 0.22$_{\text{$\pm$0.27}}$
& 0.21$_{\text{$\pm$0.27}}$
\\

ICALiNGAM
& 7.0$_{\text{$\pm$1.3}}$
& 0.87$_{\text{$\pm$0.31}}$
& 0.23$_{\text{$\pm$0.14}}$
& 0.34$_{\text{$\pm$0.18}}$
& 6.7$_{\text{$\pm$1.4}}$
& 0.89$_{\text{$\pm$0.31}}$
& 0.25$_{\text{$\pm$0.16}}$
& 0.38$_{\text{$\pm$0.21}}$
\\

DirectLiNGAM
& 6.5$_{\text{$\pm$1.4}}$
& 0.93$_{\text{$\pm$0.24}}$
& 0.28$_{\text{$\pm$0.15}}$
& 0.41$_{\text{$\pm$0.19}}$
& 6.7$_{\text{$\pm$1.5}}$
& 0.89$_{\text{$\pm$0.31}}$
& 0.26$_{\text{$\pm$0.17}}$
& 0.38$_{\text{$\pm$0.21}}$
\\

Notears
& 7.8$_{\text{$\pm$1.0}}$
& 0.76$_{\text{$\pm$0.43}}$
& 0.13$_{\text{$\pm$0.11}}$
& 0.22$_{\text{$\pm$0.16}}$
& 7.9$_{\text{$\pm$0.9}}$
& 0.72$_{\text{$\pm$0.45}}$
& 0.12$_{\text{$\pm$0.10}}$
& 0.21$_{\text{$\pm$0.16}}$
\\

GraNDAG
& 9.7$_{\text{$\pm$4.9}}$
& 0.45$_{\text{$\pm$0.38}}$
& 0.26$_{\text{$\pm$0.18}}$
& 0.31$_{\text{$\pm$0.24}}$
& 10.8$_{\text{$\pm$4.8}}$
& 0.33$_{\text{$\pm$0.26}}$
& 0.26$_{\text{$\pm$0.17}}$
& 0.28$_{\text{$\pm$0.19}}$
\\

GOLEM
& 6.7$_{\text{$\pm$1.3}}$
& \textbf{0.94}$_{\text{$\pm$\textbf{0.24}}}$
& 0.26$_{\text{$\pm$0.14}}$
& 0.39$_{\text{$\pm$0.18}}$
& 6.8$_{\text{$\pm$1.2}}$
& \textbf{0.90}$_{\text{$\pm$\textbf{0.30}}}$
& 0.25$_{\text{$\pm$0.14}}$
& 0.38$_{\text{$\pm$0.18}}$
\\

\midrule
\multicolumn{9}{c}{$d=20$} \\
\midrule

DCD
& \textbf{14.5}$_{\text{$\pm$\textbf{3.4}}}$
& 0.66$_{\text{$\pm$0.13}}$
& \textbf{0.45}$_{\text{$\pm$\textbf{0.13}}}$
& \textbf{0.53}$_{\text{$\pm$\textbf{0.13}}}$
& 15.1$_{\text{$\pm$6.4}}$
& 0.46$_{\text{$\pm$0.31}}$
& \textbf{0.46}$_{\text{$\pm$\textbf{0.31}}}$
& \textbf{0.46}$_{\text{$\pm$\textbf{0.31}}}$
\\

PC (FisherZ)
& 26.2$_{\text{$\pm$3.8}}$
& 0.11$_{\text{$\pm$0.08}}$
& 0.11$_{\text{$\pm$0.09}}$
& 0.11$_{\text{$\pm$0.08}}$
& 29.2$_{\text{$\pm$6.8}}$
& 0.11$_{\text{$\pm$0.07}}$
& 0.15$_{\text{$\pm$0.08}}$
& 0.13$_{\text{$\pm$0.08}}$
\\

PC (E-FisherZ)
& 21.0$_{\text{$\pm$3.8}}$
& 0.19$_{\text{$\pm$0.14}}$
& 0.15$_{\text{$\pm$0.10}}$
& 0.17$_{\text{$\pm$0.12}}$
& 23.5$_{\text{$\pm$6.7}}$
& 0.17$_{\text{$\pm$0.17}}$
& 0.19$_{\text{$\pm$0.16}}$
& 0.18$_{\text{$\pm$0.16}}$
\\

PC (KCIT)
& 20.0$_{\text{$\pm$2.8}}$
& 0.15$_{\text{$\pm$0.13}}$
& 0.11$_{\text{$\pm$0.09}}$
& 0.13$_{\text{$\pm$0.10}}$
& 15.6$_{\text{$\pm$4.5}}$
& 0.30$_{\text{$\pm$0.23}}$
& 0.30$_{\text{$\pm$0.22}}$
& 0.30$_{\text{$\pm$0.22}}$
\\

PC (E-KCIT)
& 17.6$_{\text{$\pm$1.9}}$
& 0.19$_{\text{$\pm$0.12}}$
& 0.13$_{\text{$\pm$0.07}}$
& 0.15$_{\text{$\pm$0.09}}$
& \textbf{14.2}$_{\text{$\pm$\textbf{3.6}}}$
& 0.33$_{\text{$\pm$0.22}}$
& 0.29$_{\text{$\pm$0.19}}$
& 0.31$_{\text{$\pm$0.20}}$
\\

GES
& 24.3$_{\text{$\pm$4.0}}$
& 0.12$_{\text{$\pm$0.16}}$
& 0.10$_{\text{$\pm$0.11}}$
& 0.11$_{\text{$\pm$0.13}}$
& 28.9$_{\text{$\pm$7.0}}$
& 0.08$_{\text{$\pm$0.08}}$
& 0.11$_{\text{$\pm$0.09}}$
& 0.09$_{\text{$\pm$0.08}}$
\\

ICALiNGAM
& 15.7$_{\text{$\pm$2.0}}$
& 0.82$_{\text{$\pm$0.33}}$
& 0.18$_{\text{$\pm$0.09}}$
& 0.29$_{\text{$\pm$0.14}}$
& 15.2$_{\text{$\pm$1.9}}$
& 0.89$_{\text{$\pm$0.24}}$
& 0.20$_{\text{$\pm$0.10}}$
& 0.32$_{\text{$\pm$0.14}}$
\\

DirectLiNGAM
& 14.9$_{\text{$\pm$1.7}}$
& 0.93$_{\text{$\pm$0.22}}$
& 0.23$_{\text{$\pm$0.09}}$
& 0.36$_{\text{$\pm$0.12}}$
& 14.5$_{\text{$\pm$2.8}}$
& 0.91$_{\text{$\pm$0.25}}$
& 0.24$_{\text{$\pm$0.15}}$
& 0.36$_{\text{$\pm$0.18}}$
\\

Notears
& 16.8$_{\text{$\pm$1.4}}$
& 0.91$_{\text{$\pm$0.28}}$
& 0.12$_{\text{$\pm$0.08}}$
& 0.20$_{\text{$\pm$0.12}}$
& 17.4$_{\text{$\pm$1.4}}$
& 0.72$_{\text{$\pm$0.46}}$
& 0.08$_{\text{$\pm$0.07}}$
& 0.14$_{\text{$\pm$0.12}}$
\\

GraNDAG
& 33.4$_{\text{$\pm$13.1}}$
& 0.23$_{\text{$\pm$0.23}}$
& 0.18$_{\text{$\pm$0.09}}$
& 0.18$_{\text{$\pm$0.10}}$
& 32.7$_{\text{$\pm$9.4}}$
& 0.18$_{\text{$\pm$0.16}}$
& 0.16$_{\text{$\pm$0.08}}$
& 0.15$_{\text{$\pm$0.09}}$
\\

GOLEM
& 14.8$_{\text{$\pm$1.5}}$
& \textbf{0.95}$_{\text{$\pm$\textbf{0.21}}}$
& 0.23$_{\text{$\pm$0.08}}$
& 0.36$_{\text{$\pm$0.11}}$
& 14.5$_{\text{$\pm$2.6}}$
& \textbf{0.96}$_{\text{$\pm$\textbf{0.20}}}$
& 0.24$_{\text{$\pm$0.14}}$
& 0.36$_{\text{$\pm$0.18}}$
\\

\bottomrule
\end{tabular}
\end{table*}

In this section, we primarily evaluate the effectiveness of our method under different data generation mechanisms.
We randomly generated DAGs by the Barabási-Albert model with an attachment parameter of 1 to create scale-free graphs (approximately degree 2 per node, i.e., in-degree or out-degree 1).
Given a DAG, the data are generated according to
\begin{equation*}
    X_i = f_i(\pa{X_i}) + \epsilon_i,
\end{equation*}
where $\epsilon_i$ is an exogenous noise randomly sampled from one of the following distributions: Gaussian, Laplace, Student's $t$, and Uniform, with a standard deviation equal to 2.
We consider two structural mechanisms $f_i(\cdot)$:
(i) \textit{Functional mechanisms (Func)}, where $f_i(\pa{X_i}) = g_i(W_i^{\top} \pa{X_i})$, with each entry of $W_i^{\top}$ independently sampled from $U(0.5,2.0)$ with random signs and divided by $\sqrt{\dim(\pa{X_i})}$, and $g_i(\cdot)$ randomly selected from $\{x, x^2, x^3, \tanh(x), \exp(x), \log(|x|)\}$.
(ii) \textit{MLP mechanisms (MLP)}, where $f_i(\pa{X_i})$ is implemented as a 5-layer MLP with hidden dimension 20 and $\tanh$ activations, with weights sampled as in Func and normalized to preserve layer-wise output variance.
For both mechanisms, the resulting structural signal $f_i(\pa{X_i})$ is standardized to zero mean and unit variance before adding noise, to prevent numerical explosion and ensure comparability with noise magnitudes.

Table~\ref{tab:main} presents the performance evaluation across two mechanisms and varying graph sizes, with sample size $n=800$, reporting the mean and standard deviation over 50 independent trials (25 trials for $d = 20$).
Across both functional and MLP mechanisms and all graph sizes, DCD achieves superior overall performance, yielding the lowest $\text{SHD}_c$ in almost all settings with consistently best Recall and F1-scores. While some baselines (e.g., GOLEM) attain high Precision at $d=20$, these overly conservative edge predictions compromise Recall, leading to suboptimal overall accuracy.

\subsection{Noise Robustness}\label{sec:expN}

\begin{figure}[t]
    \centering
    \begin{subfigure}[b]{0.98\linewidth}
        \centering
        \includegraphics[width=\linewidth]{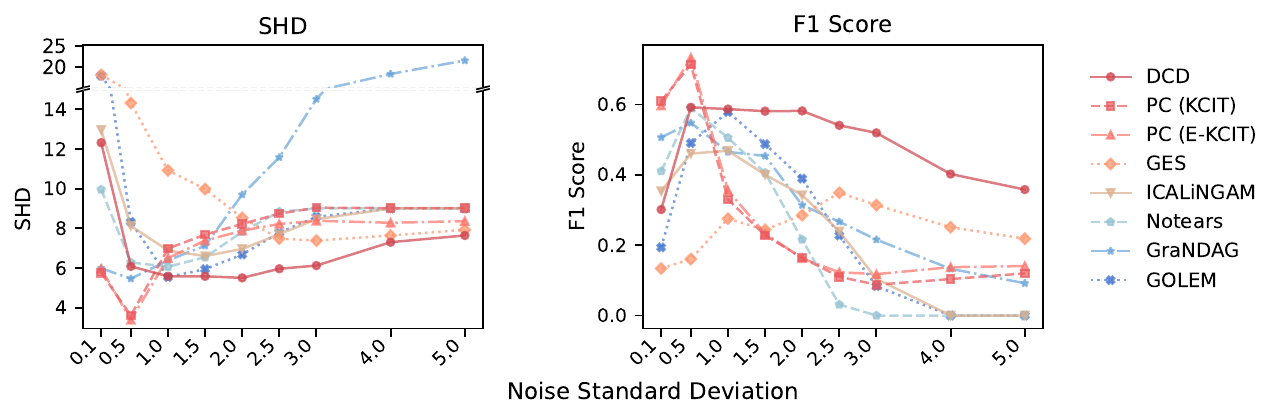}
        \caption{Functional mechanism}
        \label{fig:func}
    \end{subfigure}

    \begin{subfigure}[b]{0.98\linewidth}
        \centering
        \includegraphics[width=\linewidth]{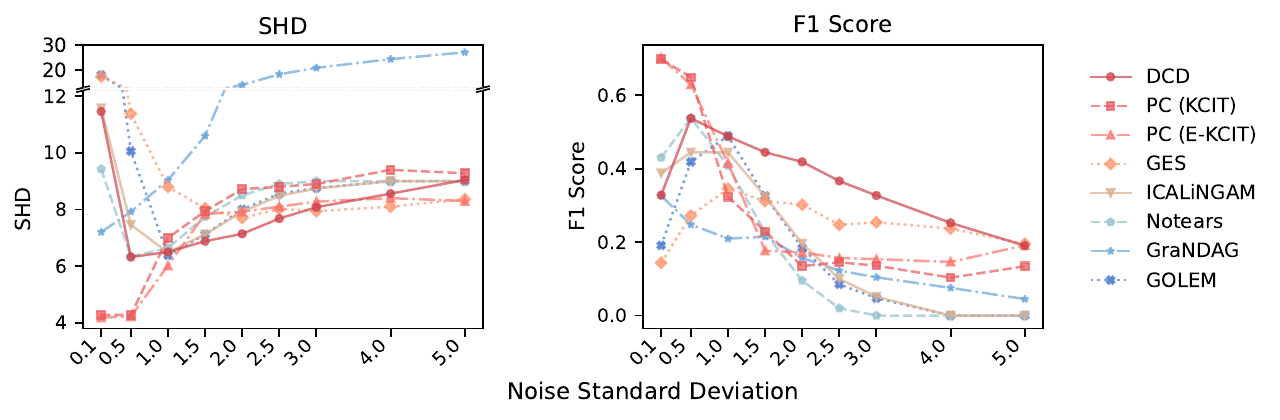}
        \caption{Noise-in-mechanism}
        \label{fig:funcH}
    \end{subfigure}
    
    \caption{Performance under increasing noise levels.}
    \label{fig:noise}
\end{figure}

We evaluate noise robustness using the same setting as in Section~\ref{sec:expM} (Func, $d=10$), while varying the standard deviation of noise from 0.1 to 5.0.
Additionally, we consider a more stringent noise-in-mechanism variant, where noise is injected before the nonlinear transformation ($X_i = g_i(W_i^{\top} \pa{X_i} + \epsilon_i)$).

Figure~\ref{fig:noise} illustrates the trajectories of $\text{SHD}_c$ (left) and F1-score (right) under two noise mechanisms, based on 50 independent trials.
For visual clarity, we omit the variance bars and lines for PC with FisherZ/E-FisherZ (which lag behind PC with KCIT/E-KCIT) and DirectLiNGAM (which is slightly inferior to ICALiNGAM) from the figures.
Under low-noise regimes, PC with KCIT/E-KCIT stands out. However, as noise increases, the performance of all baselines degrades rapidly, except GES which consistently underperforms. In contrast, DCD demonstrates superior robustness, remaining stable even under high-noise conditions. For instance, in Figure~\ref{fig:func} at noise standard deviation near $5.0$, DCD maintains an F1-score around $0.4$, whereas some baselines drop close to zero.

\subsection{Performance Across Topologies}\label{sec:expG}

\begin{figure}[t]
    \centering
    \includegraphics[width=0.95\linewidth]{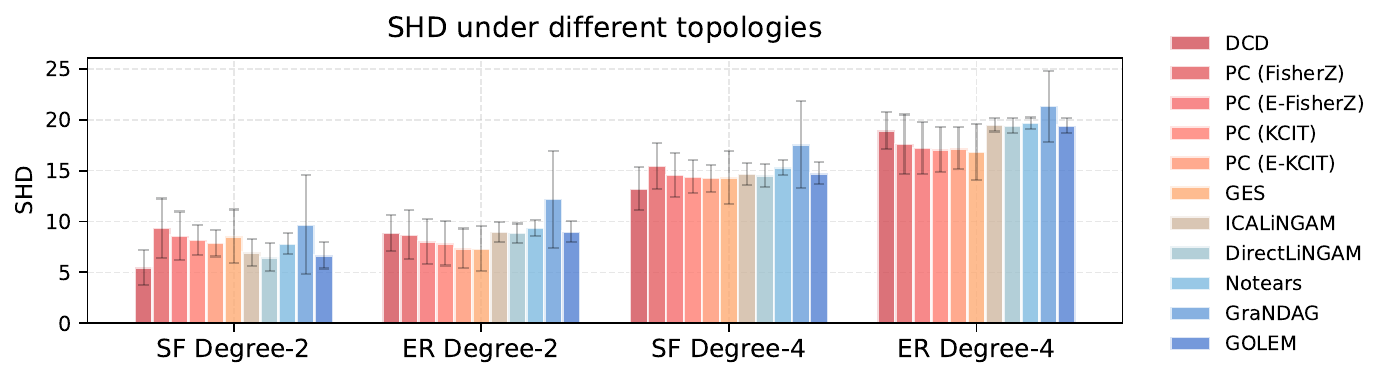}
    \caption{Performance under different graph topologies and densities.}
    \label{fig:graph}
\end{figure}

We evaluate graph topology robustness using the same setting as in Section~\ref{sec:expM} (Func, $d=10$), while varying the topological structures and densities via Scale-Free (SF) and Erdős-Rényi (ER) graphs across average degree $2$ and $4$ (i.e., 10 and 20 edges per graph).
To guarantee connectivity, SF DAGs are generated via the Barabási-Albert model with attachment parameters 1 or 2, while ER DAGs are approximately generated via Prüfer sequences and augmented with additional edges accordingly.

Figure~\ref{fig:graph} presents the $\text{SHD}_c$ under four different graph settings, averaged over 50 trials.
Error bars represent the standard deviation across the independent trials.
DCD exhibits solid robustness with a substantial advantage across SF graphs of both densities. Although PC and GES show superior performance on ER graphs, DCD remains highly competitive and outperforms all other baselines.

\subsection{Sample Efficiency}\label{sec:expS}

\begin{figure}[t]
    \centering
    \includegraphics[width=0.98\linewidth]{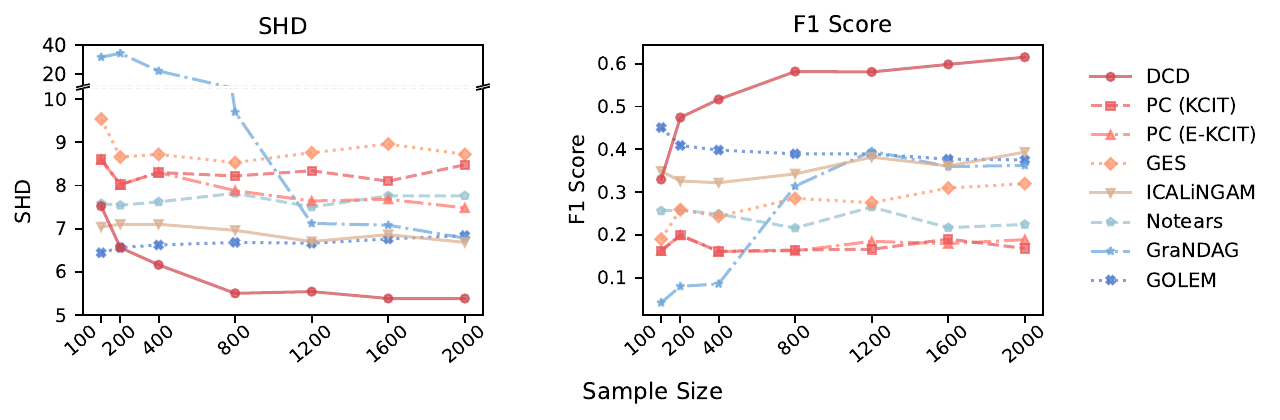}
    \caption{Performance under increasing sample sizes.}
    \label{fig:sample}
\end{figure}

We evaluate sample efficiency using the same setting as in Section~\ref{sec:expM} (Func, $d=10$), while varying the sample size from 100 to 2000.
Figure~\ref{fig:sample} illustrates the trajectories over sample sizes (with the same plotting conventions as Figure~\ref{fig:noise}).
DCD outperforms all baselines across almost all sample sizes, with its performance steadily improving as the sample size increases.

\section{Discussion}\label{sec:discussion}

In this paper, we have introduced Decoupled Causal Discovery (DCD) and established its theoretical foundation.
By leveraging the properties of decoupled distributions, DCD directly identifies the MB and iteratively constructs the global structure. 
Although DCD has demonstrated empirical effectiveness, there are several limitations and future directions, 
including the practical challenges of learning decoupling weights (e.g., the requirement of a certain amount of noise, high dimensionality, and the risk of information loss) and the potential of the purely MB-driven Algorithm~\ref{alg:dcd-c} (especially combining decoupling with CIT).
Due to space limitations, detailed discussions are provided in Appendix~\ref{app:future}.
Overall, as a decoupling-based perspective orthogonal to CITs, score functions, or restrictive functional assumptions, DCD broadens the methodological scope of causal discovery.


\subsection*{AI use statement}

In this work, we used generative AI tools for assisting with translation and assisting in the writing of proofs.
We have not used generative AI tools for helping develop theoretical models or conceptual frameworks, or any other tasks with required disclosure.
Additionally, we used generative AI tools for editing the paper to improve readability.
We have reviewed all AI-assisted work. All revised text was manually checked for clarity and technical accuracy.
We take responsibility for the final content of this work, including text, claims or artifacts produced with the aid of generative AI.

\subsection*{Reproducibility statement}

We provide a complete and anonymous source code package containing all algorithm implementations and experiment scripts in the Supplementary Material.
Additionally, the main text includes the detailed pseudo-code for our algorithm (Section~\ref{subsec:global}), implementation details (Section~\ref{subsec:practical_e}), and comprehensive experimental settings (Section~\ref{sec:exp}).
Complete proofs for all theoretical results in the main text are provided in Appendix~\ref{app:proofs}.



\bibliography{references}
\bibliographystyle{iclr2027_conference}

\newpage
\appendix


\renewcommand{\thetheorem}{\thesection.\arabic{theorem}}
\renewcommand{\thecorollary}{\thesection.\arabic{corollary}}
\renewcommand{\thedefinition}{\thesection.\arabic{definition}}
\renewcommand{\theproposition}{\thesection.\arabic{proposition}}
\renewcommand{\thelemma}{\thesection.\arabic{lemma}}
\setcounter{theorem}{0}
\setcounter{corollary}{0}
\setcounter{definition}{0}
\setcounter{proposition}{0}
\setcounter{lemma}{0}

\section{Omitted Proofs}\label{app:proofs}

To facilitate the subsequent derivations, we first recall several fundamental properties of conditional independence (i.e., the graphoid axioms) and the Markov boundary (MB)~\citep{pearl1988}.

\begin{lemma}[Decomposition Axiom]\label{lemma:decomposition}
For any disjoint sets of random variables $\mathbf{U}$, $\mathbf{V}$, $\mathbf{W}$, and $\mathbf{T}$:
\begin{equation*}
\mathbf{U} \perp\!\!\!\perp (\mathbf{V} \cup \mathbf{T}) \mid \mathbf{W} \implies \mathbf{U} \perp\!\!\!\perp \mathbf{V} \mid \mathbf{W} \;\; \text{and} \;\; \mathbf{U} \perp\!\!\!\perp \mathbf{T} \mid \mathbf{W}.
\end{equation*}
\end{lemma}

\begin{lemma}[Weak Union Axiom]\label{lemma:weak_union}
For any disjoint sets of random variables $\mathbf{U}$, $\mathbf{V}$, $\mathbf{W}$, and $\mathbf{T}$:
\begin{equation*}
\mathbf{U} \perp\!\!\!\perp (\mathbf{V} \cup \mathbf{T}) \mid \mathbf{W} \implies \mathbf{U} \perp\!\!\!\perp \mathbf{T} \mid (\mathbf{W} \cup \mathbf{V}).
\end{equation*}
\end{lemma}

\begin{lemma}[Contraction Axiom]\label{lemma:contraction}
For any disjoint sets of random variables $\mathbf{U}$, $\mathbf{V}$, $\mathbf{W}$, and $\mathbf{T}$:
\begin{equation*}
\mathbf{U} \perp\!\!\!\perp \mathbf{V} \mid \mathbf{W} \;\; \text{and} \;\; \mathbf{U} \perp\!\!\!\perp \mathbf{T} \mid (\mathbf{W} \cup \mathbf{V}) \implies \mathbf{U} \perp\!\!\!\perp (\mathbf{V} \cup \mathbf{T}) \mid \mathbf{W}.
\end{equation*}
\end{lemma}

\begin{lemma}[Intersection Axiom]\label{lemma:intersection}
Under Assumption~\ref{ass:pos}, for any disjoint sets of random variables $\mathbf{U}$, $\mathbf{V}$, $\mathbf{W}$, and $\mathbf{T}$:
\begin{equation*}
\mathbf{U} \perp\!\!\!\perp \mathbf{V} \mid (\mathbf{W} \cup \mathbf{T}) \;\; \text{and} \;\; \mathbf{U} \perp\!\!\!\perp \mathbf{T} \mid (\mathbf{W} \cup \mathbf{V}) \implies \mathbf{U} \perp\!\!\!\perp (\mathbf{V} \cup \mathbf{T}) \mid \mathbf{W}.
\end{equation*}
\end{lemma}

\begin{lemma}[Uniqueness of MB]\label{lemma:mb_uniqueness}
Let $\mathcal{O}$ be a set of random variables. 
Under Assumption~\ref{ass:pos}, for any variable $X_i \in \mathcal{O}$, $\MB{X_i \mid \mathcal{O}}$ is unique.
\end{lemma}

\begin{proof}
We prove this by contradiction. Suppose there exist two distinct MBs for $X_i$, denoted by $\mathbf{M}_1$ and $\mathbf{M}_2$, with $\mathbf{M}_1 \neq \mathbf{M}_2$. By Definition~\ref{def:mb}, we have:
\begin{equation*}
X_i \perp\!\!\!\perp \bigl( \mathcal{O} \setminus (\{X_i\} \cup \mathbf{M}_1) \bigr) \mid \mathbf{M}_1
\quad \text{and} \quad
X_i \perp\!\!\!\perp \bigl( \mathcal{O} \setminus (\{X_i\} \cup \mathbf{M}_2) \bigr) \mid \mathbf{M}_2.
\end{equation*}
Let $\mathbf{D} = \mathbf{M}_1 \cap \mathbf{M}_2$, $\mathbf{V}_1 = \mathbf{M}_1 \setminus \mathbf{D}$, $\mathbf{V}_2 = \mathbf{M}_2 \setminus \mathbf{D}$, and $\mathbf{R} = \mathcal{O} \setminus (\{X_i\} \cup \mathbf{M}_1 \cup \mathbf{M}_2)$. 
The two conditional independence statements can be rewritten respectively as:
\begin{equation*}
X_i \perp\!\!\!\perp (\mathbf{V}_2 \cup \mathbf{R}) \mid (\mathbf{D} \cup \mathbf{V}_1)
\quad \text{and} \quad
X_i \perp\!\!\!\perp (\mathbf{V}_1 \cup \mathbf{R}) \mid (\mathbf{D} \cup \mathbf{V}_2).
\end{equation*}
By the decomposition axiom (Lemma~\ref{lemma:decomposition}), this implies:
\begin{equation*}
X_i \perp\!\!\!\perp \mathbf{V}_2 \mid (\mathbf{D} \cup \mathbf{V}_1)
\quad \text{and} \quad
X_i \perp\!\!\!\perp \mathbf{V}_1 \mid (\mathbf{D} \cup \mathbf{V}_2).
\end{equation*}
Applying the intersection axiom (Lemma~\ref{lemma:intersection}) yields:
\begin{equation}
\label{eq:un1}
X_i \perp\!\!\!\perp (\mathbf{V}_1 \cup \mathbf{V}_2) \mid \mathbf{D}.
\end{equation}
Meanwhile, by applying the weak union axiom (Lemma~\ref{lemma:weak_union}) to $X_i \perp\!\!\!\perp (\mathbf{V}_2 \cup \mathbf{R}) \mid (\mathbf{D} \cup \mathbf{V}_1)$, we obtain:
\begin{equation}
\label{eq:un2}
X_i \perp\!\!\!\perp \mathbf{R} \mid (\mathbf{D} \cup \mathbf{V}_1 \cup \mathbf{V}_2).
\end{equation}
By the contraction axiom (Lemma~\ref{lemma:contraction}), combining equations~\eqref{eq:un1}~and~\eqref{eq:un2} yields:
\begin{equation*}
X_i \perp\!\!\!\perp (\mathbf{V}_1 \cup \mathbf{V}_2 \cup \mathbf{R}) \mid \mathbf{D}.
\end{equation*}
Since $\mathbf{V}_1 \cup \mathbf{V}_2 \cup \mathbf{R} = \mathcal{O} \setminus (\{X_i\} \cup \mathbf{D})$, this is equivalent to $X_i \perp\!\!\!\perp \bigl( \mathcal{O} \setminus (\{X_i\} \cup \mathbf{D}) \bigr) \mid \mathbf{D}$.
Thus, by Definition~\ref{def:mb}, $\mathbf{D} = \mathbf{M}_1 \cap \mathbf{M}_2$ is a Markov blanket of $X_i$. Since $\mathbf{M}_1 \neq \mathbf{M}_2$, we have $\mathbf{D} \subsetneq \mathbf{M}_1$, which contradicts the minimality of $\mathbf{M}_1$ (Definition~\ref{def:mb}). Therefore, the MB is unique.
\end{proof}

While Lemmas~\ref{lemma:decomposition}--\ref{lemma:intersection} are standard axioms~\citep{pearl1988}, we provide the brief proof for Lemma~\ref{lemma:mb_uniqueness} as a direct corollary, since the strict uniqueness of the MB is a foundational prerequisite for our theoretical framework. Beyond these purely probabilistic properties, incorporating a graph structure yields the following characterization \citep{pearl1988}:

\begin{lemma}[MB for DAGs]\label{lemma:mb}
Under Assumptions~\ref{ass:1}, \ref{ass:2}, and~\ref{ass:3}, for the MB of $X_i$ in $\mathcal{G} = (\mathcal{V}, \mathcal{E})$, we have:
\begin{equation*}
\MB{X_i \mid \mathcal{V}} = \pa{X_i} \cup \ch{X_i} \cup \spn{X_i}.
\end{equation*}
\end{lemma}

In the following subsections, we provide detailed proofs for Theorems \ref{the:MB} and \ref{the:AG}, respectively.

\subsection{MB Identifiability}\label{app:mb_proof}

We first prove Proposition~\ref{pro:inv}, a general property of the weighted distribution.

\begingroup
\renewcommand{\theproposition}{\ref{pro:inv}}
\begin{proposition}[MB Conditional Distribution Invariance]
Let $\mathcal{O}$ be a set of variables, $X_t \in \mathcal{O}$ be the target variable, and $\mathbf{X}_{-t} = \mathcal{O} \setminus \{X_t\}$. Under Assumption~\ref{ass:pos}, for any weighting function $w \in \mathcal{W}(\mathbf{X}_{-t})$ (and its corresponding $\tilde{P}_w(X_t, \mathbf{X}_{-t})$), the following holds for any $\mathbf{X}_{\textup{MB}}^{+}$ satisfying $\MB{X_t \mid \mathcal{O}} \subseteq \mathbf{X}_{\textup{MB}}^{+} \subseteq \mathbf{X}_{-t}$:
\begin{equation*}
\tilde{P}_w(X_t \mid \mathbf{X}_{\textup{MB}}^+)
=
P(X_t \mid \mathbf{X}_{\textup{MB}}^+)
=
P(X_t \mid \MB{X_t \mid \mathcal{O}})
=
P(X_t \mid \mathbf{X}_{-t})
.
\end{equation*}

\end{proposition}
\endgroup

\begin{proof}

By the definition of MB (Definition~\ref{def:mb}), it is straightforward to obtain that
\begin{equation*}
P(X_t \mid \mathbf{X}_{\textup{MB}}^+) = P(X_t \mid \MB{X_t \mid \mathcal{O}}) = P(X_t \mid \mathbf{X}_{-t}).
\end{equation*}
Thus, we need only establish that $\tilde{P}_w(X_t \mid \mathbf{X}_{\textup{MB}}^+) = P(X_t \mid \mathbf{X}_{\textup{MB}}^+) $. To simplify the integration, we denote
$\mathbf{X}_{\textup{MB}} = \MB{X_t \mid \mathcal{O}}$, 
$\mathbf{R}_t = \mathcal{O} \setminus \mathbf{X}_{\textup{MB}}^+$ and
$\mathbf{R}_{-t} = \mathbf{X}_{-t} \setminus \mathbf{X}_{\textup{MB}}^+$. Then we have:
\begin{equation*}
\begin{aligned}
\tilde{P}_w(X_t \mid \mathbf{X}_{\textup{MB}}^{+})
& = \frac{\tilde{P}_w(X_t,\mathbf{X}_{\textup{MB}}^{+})}{\tilde{P}_w(\mathbf{X}_\textup{MB}^{+})} \\
& = \frac{
\displaystyle \int\tilde{P}_w(X_t,\mathbf{X}_{-t}) \, d \mathbf{R}_{-t}
}{
\displaystyle \int\tilde{P}_w(X_t,\mathbf{X}_{-t}) \, d \mathbf{R}_{t}
} \\ 
& \overset{\text{(a)}}{=} \frac{
\displaystyle \int w(\mathbf{X}_{-t}) P(X_t,\mathbf{X}_{-t})  \, d \mathbf{R}_{-t}
}{
\displaystyle \int w(\mathbf{X}_{-t}) P(X_t,\mathbf{X}_{-t})  \, d \mathbf{R}_{t}
}\\ 
& = \frac{
\displaystyle \int w(\mathbf{X}_{-t}) P(X_t \mid \mathbf{X}_{-t}) P(\mathbf{X}_{-t})  \, d \mathbf{R}_{-t}
}{
\displaystyle \int w(\mathbf{X}_{-t}) P(\mathbf{X}_{-t})  \, d \mathbf{R}_{-t}
} \\
& \overset{\text{(b)}}{=} \frac{
\displaystyle \int w(\mathbf{X}_{-t}) P(X_t \mid \mathbf{X}_{\textup{MB}}) P(\mathbf{X}_{-t})  \, d \mathbf{R}_{-t}
}{
\displaystyle \int w(\mathbf{X}_{-t}) P(\mathbf{X}_{-t})  \, d \mathbf{R}_{-t}
} \\
& \overset{\text{(c)}}{=} \frac{
P(X_t \mid \mathbf{X}_{\textup{MB}})
\displaystyle \int w(\mathbf{X}_{-t})  P(\mathbf{X}_{-t})  \, d \mathbf{R}_{-t}
}{
\displaystyle \int w(\mathbf{X}_{-t}) P(\mathbf{X}_{-t})  \, d \mathbf{R}_{-t}
} \\ 
& = P(X_t \mid \mathbf{X}_{\textup{MB}})
\end{aligned}
\end{equation*}
where equation (a) follows from the definition of $\tilde{P}_w$ in Definition~\ref{def:w}, equation (b) follows from the definition of MB (Definition~\ref{def:mb}), and (c) follows from the fact that $P(X_t \mid \mathbf{X}_{\textup{MB}})$ does not depend on $\mathbf{R}_{-t}$.
\end{proof}

Building on this property, together with the definition of the decoupled distribution (Definition~\ref{def:decoupled}) and Weighting Faithfulness (Assumption~\ref{ass:wf}), we prove Theorem~\ref{the:MB}.

\begingroup
\renewcommand{\thetheorem}{\ref{the:MB}}
\begin{theorem}[MB Identifiability]
Let $\mathcal{O}$ be a set of variables, $X_t \in \mathcal{O}$ be the target variable, and $\mathbf{X}_{-t} = \mathcal{O} \setminus \{X_t\}$.
Under Assumption~\ref{ass:pos}, for any decoupling weight $w \in \mathcal{W}_{\perp}(\mathbf{X}_{-t})$ (and its corresponding $\tilde{P}_w(X_t, \mathbf{X}_{-t})$) satisfying Assumption~\ref{ass:wf}, we have for any variable $X_i \in \mathbf{X}_{-t}$:
\begin{equation*}
    X_i \in \MB{X_t \mid \mathcal{O}} \iff X_i \not\!\perp\!\!\!\perp_{\tilde{P}_w}\!\! X_t.
\end{equation*}
\end{theorem}
\endgroup

\begin{proof}

To prove the equivalence $X_i \in \MB{X_t \mid \mathcal{O}} \iff X_i \not\!\perp\!\!\!\perp_{\tilde{P}_w}\!\! X_t$, we show that both the necessity and sufficiency hold. For notational convenience, let $\mathbf{X}_{\textup{MB}} = \MB{X_t \mid \mathcal{O}}$ and $\mathbf{X}_{\textup{MB} \setminus i} = \MB{X_t \mid \mathcal{O}} \setminus \{X_i\}$.

\paragraph{\textit{Necessity.}}
We show that for every $X_i \in \MB{X_t \mid \mathcal{O}}$, it holds that $X_i \not\!\perp\!\!\!\perp_{\tilde{P}_w}\!\! X_t$, which is equivalent to showing that the conditional distribution $\tilde{P}_w(X_t \mid X_i)$ is a nontrivial function of $X_i$. We can reformulate this conditional distribution:
\begin{equation*}
\begin{aligned}
\tilde{P}_w(X_t \mid X_i)
&= \int \tilde{P}_w(X_t, \mathbf{X}_{\textup{MB} \setminus i} \mid X_i) \, d\mathbf{X}_{\textup{MB} \setminus i} \\
&= \int \tilde{P}_w(X_t \mid \mathbf{X}_{\textup{MB} \setminus i}, X_i) \tilde{P}_w(\mathbf{X}_{\textup{MB} \setminus i} \mid X_i) \, d\mathbf{X}_{\textup{MB} \setminus i} \\
&\overset{\text{(a)}}{=} \int \tilde{P}_w(X_t \mid \mathbf{X}_{\textup{MB}}) \tilde{P}_w(\mathbf{X}_{\textup{MB} \setminus i}) \, d\mathbf{X}_{\textup{MB} \setminus i} \\
&\overset{\text{(b)}}{=} \int P(X_t \mid \mathbf{X}_{\textup{MB}}) \tilde{P}_w(\mathbf{X}_{\textup{MB} \setminus i}) \, d\mathbf{X}_{\textup{MB} \setminus i}
\end{aligned}
\end{equation*}
where equation (a) follows from $X_i  \perp\!\!\!\perp_{\tilde{P}_w}\!\! \mathbf{X}_{\textup{MB} \setminus i}$ by Definition~\ref{def:decoupled}, and equation (b) follows from Proposition~\ref{pro:inv}.

We next prove $X_t \not\!\perp\!\!\!\perp X_i \mid \mathbf{X}_{\textup{MB} \setminus i}$ by contradiction. Assume the contrary:
\begin{equation}
X_t \perp\!\!\!\perp X_i \mid \mathbf{X}_{\textup{MB} \setminus i}.
\label{eq:indep}
\end{equation}
Meanwhile, by the definition of MB (Definition~\ref{def:mb}), we have
\begin{equation}
X_t \perp\!\!\!\perp \mathbf{R} \mid (\mathbf{X}_{\textup{MB} \setminus i} \cup \{X_i\}),
\label{eq:mbdef}
\end{equation}
where $\mathbf{R} = \mathcal{O} \setminus (\mathbf{X}_{\textup{MB}} \cup \{X_t\})$.
Applying the contraction axiom (Lemma~\ref{lemma:contraction}) to combine equations~\eqref{eq:indep} and \eqref{eq:mbdef} yields:
\begin{equation*}
X_t \perp\!\!\!\perp (\mathbf{R} \cup \{X_i\}) \mid \mathbf{X}_{\textup{MB} \setminus i},
\end{equation*}
which is a contradiction to the definition of MB (inclusion-minimal, Definition~\ref{def:mb}). 
Therefore, we conclude that $X_t \not\!\perp\!\!\!\perp X_i \mid \mathbf{X}_{\textup{MB} \setminus i}$, which implies that $P(X_t \mid X_i, \mathbf{X}_{\textup{MB} \setminus i})$ depends on $X_i$.

Note that
\begin{equation*}
\begin{aligned}
\tilde{P}_w(X_t \mid X_i)
& = \int P(X_t \mid \mathbf{X}_{\textup{MB}}) \tilde{P}_w(\mathbf{X}_{\textup{MB} \setminus i}) \, d\mathbf{X}_{\textup{MB} \setminus i} \\
& = \mathbb{E}_{\tilde{P}_w(\mathbf{X}_{\textup{MB} \setminus i})}[P(X_t \mid X_i, \mathbf{X}_{\textup{MB} \setminus i})].
\end{aligned}
\end{equation*}
Therefore, by Assumption~\ref{ass:wf}, $\tilde{P}_w(X_t \mid X_i)$ also depends on $X_i$.
Thus, the proof of necessity is complete.

\paragraph{\textit{Sufficiency.}}
We prove this by contraposition, i.e., showing $X_i \not\in \MB{X_t \mid \mathcal{O}} \implies X_i \perp\!\!\!\perp_{\tilde{P}_w}\!\! X_t$.
Equivalently, we show $\tilde{P}_w(X_t \mid X_i) = \tilde{P}_w(X_t)$:
\begin{equation*}
\begin{aligned}
\tilde{P}_w(X_t \mid X_i) &= \int \tilde{P}_w(X_t, \mathbf{X}_{\textup{MB}} \mid X_i) \, d\mathbf{X}_{\textup{MB}} \\
&= \int \tilde{P}_w(X_t \mid \mathbf{X}_{\textup{MB}}, X_i) \tilde{P}_w(\mathbf{X}_{\textup{MB}} \mid X_i) \, d\mathbf{X}_{\textup{MB}} \\
&\overset{\text{(a)}}{=} \int P(X_t \mid \mathbf{X}_{\textup{MB}}, X_i) \tilde{P}_w(\mathbf{X}_{\textup{MB}}) \, d\mathbf{X}_{\textup{MB}} \\
&\overset{\text{(b)}}{=} \int P(X_t \mid \mathbf{X}_{\textup{MB}}) \tilde{P}_w(\mathbf{X}_{\textup{MB}}) \, d\mathbf{X}_{\textup{MB}} \\
&\overset{\text{(c)}}{=} \int \tilde{P}_w(X_t \mid \mathbf{X}_{\textup{MB}}) \tilde{P}_w(\mathbf{X}_{\textup{MB}}) \, d\mathbf{X}_{\textup{MB}} \\
&= \tilde{P}_w(X_t) \\
\end{aligned}
\end{equation*}
where equation (a) follows from Proposition~\ref{pro:inv} and $X_i  \perp\!\!\!\perp_{\tilde{P}_w}\!\! \mathbf{X}_{\textup{MB}}$ by Definition~\ref{def:decoupled}, equation (b) follows from the definition of MB (Definition~\ref{def:mb}), and (c) follows from Proposition~\ref{pro:inv}. Thus, the proof of sufficiency is complete.

With both sufficiency and necessity proved, the equivalence follows.

\end{proof}

Finally, we show that Weighting Faithfulness (Assumption~\ref{ass:wf}) required by Theorem~\ref{the:MB} holds almost everywhere under Condition~\ref{ass:nd}.

\begingroup
\renewcommand{\theproposition}{\ref{pro:gwf}}
\begin{proposition}[Genericity of Weighting Faithfulness]
Let $\mathcal{O}$ be a set of variables, $X_t \in \mathcal{O}$ be the target variable, and $\mathbf{X}_{-t} = \mathcal{O} \setminus \{X_t\}$.
Under Assumption~\ref{ass:pos} and Condition~\ref{ass:nd}, for $\mu_{\mathcal{W}}^{\mathbf{X}_{-t}}$-almost every weight $w \in \mathcal{W}_{\perp}(\mathbf{X}_{-t})$, Assumption~\ref{ass:wf} (Weighting Faithfulness) holds.
\end{proposition}
\endgroup

\begin{proof}

Consider any $X_e \in \mathbf{X}_{-t}$ and
$\mathbf{X}_c \subseteq \mathbf{X}_{-t} \setminus \{X_e\}$.
We prove that if $P(X_t \mid X_e, \mathbf{X}_c)$ depends on $X_e$, then
\begin{equation*}
g(X_t, X_e)
=
\mathbb{E}_{\tilde{P}_w(\mathbf{X}_c)}
[P(X_t \mid X_e, \mathbf{X}_c)]
\end{equation*}
is not a constant function of $X_e$, for
$\mu_{\mathcal{W}}^{\mathbf{X}_{-t}}$-almost every weight
$w \in \mathcal{W}_{\perp}(\mathbf{X}_{-t})$.

Let $\nu = P_{X_e}$.
For measurable events $A$ and $B$ in the state spaces of $X_t$ and $X_e$, respectively, define
\begin{equation*}
h_A(x_e,\mathbf{x}_c)
=
P(X_t \in A \mid X_e=x_e,\mathbf{X}_c=\mathbf{x}_c),
\end{equation*}
and
\begin{equation*}
r_B(x_e)=\mathbf{1}_B(x_e)-\nu(B).
\end{equation*}

Since $P(X_t \mid X_e,\mathbf{X}_c)$ depends on $X_e$, there exist such events $A$ and $B$ for which
\begin{equation*}
\Delta(\mathbf{x}_c)
:=
\int r_B(x_e)h_A(x_e,\mathbf{x}_c)\,\nu(dx_e)
\end{equation*}
is nonzero on a set of positive $\lambda_{\mathbf{X}_c}$-measure.

We prove this by contradiction.
Consider countable determining classes of events for $X_t$ and $X_e$ (for real-valued variables, rational half-lines suffice to determine the distribution).
If the above integral were zero almost everywhere for every pair $A, B$ from these classes, then outside a common null set of $\mathbf{x}_c$ values, we would have
\begin{equation*}
\int_B h_A(x_e,\mathbf{x}_c)\,\nu(dx_e)
=
\nu(B)\int h_A(x_e,\mathbf{x}_c)\,\nu(dx_e)
\end{equation*}
for every such pair.
By uniqueness of finite measures, $h_A(\cdot,\mathbf{x}_c)$ would be $\nu$-almost everywhere constant for every $A$ in the determining class.
Under Assumption~\ref{ass:pos}, the joint distribution of $(X_e,\mathbf{X}_c)$ is equivalent to $\nu \otimes \lambda_{\mathbf{X}_c}$.
Consequently, $P(X_t \mid X_e,\mathbf{X}_c)$ would depend only on $\mathbf{X}_c$, which contradicts the fact that $P(X_t \mid X_e,\mathbf{X}_c)$ depends on $X_e$.

Fix such events $A$ and $B$.
Since $0 \leq h_A \leq 1$ and $|r_B| \leq 1$, $\Delta(\mathbf{x}_c)$ is bounded and measurable, with $|\Delta(\mathbf{x}_c)| \leq 1$.
Moreover,
\begin{equation*}
\int r_B(x_e)\,\nu(dx_e)=0.
\end{equation*}

If some $w' \in \mathcal{W}_{\perp}(\mathbf{X}_{-t})$ renders $g(X_t,X_e)$ constant as a function of $X_e$, then
\begin{equation*}
g_A(x_e)
:=
\int h_A(x_e,\mathbf{x}_c)
\,\tilde{P}_{w'}(\mathbf{x}_c)
\,\lambda_{\mathbf{X}_c}(d\mathbf{x}_c)
\end{equation*}
is also constant $\nu$-almost everywhere.
Therefore, by Fubini's theorem,
\begin{equation*}
\begin{aligned}
0
&=
\int r_B(x_e)g_A(x_e)\,\nu(dx_e)\\
&=
\int
\left(
\int r_B(x_e)h_A(x_e,\mathbf{x}_c)\,\nu(dx_e)
\right)
\tilde{P}_{w'}(\mathbf{x}_c)
\,\lambda_{\mathbf{X}_c}(d\mathbf{x}_c)\\
&=
\int \Delta(\mathbf{x}_c)
\,\tilde{P}_{w'}(\mathbf{x}_c)
\,\lambda_{\mathbf{X}_c}(d\mathbf{x}_c).
\end{aligned}
\end{equation*}

Using the relation
\begin{equation*}
\tilde{P}_{w'}(\mathbf{x}_c)
=
\int \tilde{P}_{w'}(\mathbf{x}_{-t})
\,\lambda_{\mathbf{R}_c}(d\mathbf{r}_c),
\end{equation*}
where $\mathbf{r}_c$ denotes a value of
$\mathbf{R}_c=\mathbf{X}_{-t}\setminus\mathbf{X}_c$,
we obtain
\begin{equation*}
\begin{aligned}
\int \Delta(\mathbf{x}_c)
\,\tilde{P}_{w'}(\mathbf{x}_c)
\,\lambda_{\mathbf{X}_c}(d\mathbf{x}_c)
&=
\int \Delta(\mathbf{x}_c)
\,\tilde{P}_{w'}(\mathbf{x}_{-t})
\,\lambda_{\mathbf{X}_{-t}}(d\mathbf{x}_{-t})\\
&=0.
\end{aligned}
\end{equation*}

Define the linear functional
$G:L^1(\lambda_{\mathbf{X}_{-t}})\to\mathbb{R}$ by
\begin{equation*}
G(f)
=
\int \Delta(\mathbf{x}_c)f(\mathbf{x}_{-t})
\,\lambda_{\mathbf{X}_{-t}}(d\mathbf{x}_{-t}).
\end{equation*}
Thus, the set of $w \in \mathcal{W}_{\perp}(\mathbf{X}_{-t})$ for which $g(X_t,X_e)$ is constant as a function of $X_e$ is contained in the set of weights satisfying
\begin{equation*}
\begin{aligned}
G(w \cdot P_{\mathbf{X}_{-t}})
&=
\int \Delta(\mathbf{x}_c)
\,w(\mathbf{x}_{-t})P(\mathbf{x}_{-t})
\,\lambda_{\mathbf{X}_{-t}}(d\mathbf{x}_{-t})\\
&=0.
\end{aligned}
\end{equation*}
Since $|\Delta(\mathbf{x}_c)| \leq 1$, we have $|G(f)|\leq\|f\|_1$, so $G$ is continuous.
Moreover, because $\Delta(\mathbf{x}_c)$ is nonzero on a set of positive $\lambda_{\mathbf{X}_c}$-measure and $\lambda_{\mathbf{X}_{-t}}$ is the corresponding product measure, $G$ is not the zero functional.
Hence, by Condition~\ref{ass:nd},
\begin{equation*}
\mu_{\mathcal{W}}^{\mathbf{X}_{-t}}
\bigl(
\bigl\{
w \in \mathcal{W}_{\perp}(\mathbf{X}_{-t})
\;\big\vert\;
G(w \cdot P_{\mathbf{X}_{-t}})=0
\bigr\}
\bigr)
=0.
\end{equation*}
Therefore,
\begin{equation*}
\mu_{\mathcal{W}}^{\mathbf{X}_{-t}}
\bigl(
\bigl\{
w \in \mathcal{W}_{\perp}(\mathbf{X}_{-t})
\;\big\vert\;
g(X_t,X_e)
\text{ is constant as a function of }X_e
\bigr\}
\bigr)
=0.
\end{equation*}
Since there are finitely many choices of $X_e$ and $\mathbf{X}_c$, the union of the corresponding exceptional sets also has measure zero.

\end{proof}

\subsection{Soundness and Completeness of DCD-C}\label{app:ag_proof}

We prove Theorem~\ref{the:AG} in several steps.
Lemmas~\ref{np} and~\ref{sp} establish basic properties of MBs in DAGs.
Based on these, Lemma~\ref{suff} gives a sufficient condition for v-structure exclusion, and Corollary~\ref{ext} further demonstrates the existence of such a condition.
We then prove the soundness and completeness of the skeleton and v-structures (Lemmas~\ref{sck} and~\ref{scv}), thereby establishing Theorem~\ref{the:AG}.

Since our setting assumes causal sufficiency (Assumption~\ref{ass:1}), we adopt a direct and self-contained proof strategy based on the statistical definition of MB (Definition~\ref{def:mb}) for intuitive simplicity, without introducing Maximal Ancestral Graphs (MAGs) or latent projections~\citep{MAG}.
All graphical relations we discussed (e.g., $\pa{\cdot}$, $\ch{\cdot}$, and $\spn{\cdot}$) are defined with respect to $\mathcal{G}$.
If the MAG framework were introduced, the overall structure would be similar, while some of the arguments could instead be formulated based on corresponding graphical tools such as inducing paths.

\begin{lemma}[Neighbor Persistence]\label{np}
Let $\mathcal{O} \subseteq \mathcal{V}$ be a set of variables containing $X_i$ and a neighbor $A \in \pa{X_i} \cup \ch{X_i}$. Then $A \in \MB{X_i \mid \mathcal{O}}$.
\end{lemma}

\begin{proof}
We prove this by contradiction.
Assuming $A \notin \MB{X_i \mid \mathcal{O}}$, by the definition of the MB (Definition~\ref{def:mb}) and the decomposition axiom (Lemma~\ref{lemma:decomposition}), we have
\begin{equation*}
X_i \perp\!\!\!\perp A \mid \MB{X_i \mid \mathcal{O}},
\end{equation*}
with $\MB{X_i \mid \mathcal{O}} \subseteq \mathcal{O} \subseteq \mathcal{V}$.

Since $X_i$ and $A$ are adjacent in $\mathcal{G}$, they cannot be d-separated.
Therefore, by Assumption~\ref{ass:3} (Faithfulness), $X_i$ and $A$ cannot be conditionally independent under any conditioning set.
This contradicts $X_i \perp\!\!\!\perp A \mid \MB{X_i \mid \mathcal{O}}$. Hence, $A \in \MB{X_i \mid \mathcal{O}}$.

\end{proof}


\begin{lemma}[Spouse Persistence]\label{sp}
Let $\mathcal{O} \subseteq \mathcal{V}$ be a set of variables containing $X_i$, $B \in \spn{X_i}$ and $A \in \ch{X_i} \cap \ch{B}$. Then $B \in \MB{X_i \mid \mathcal{O}}$.
\end{lemma}

\begin{proof}
If $B$ and $X_i$ are adjacent, the statement reduces to Lemma~\ref{np} and holds trivially. We therefore consider the case where $B$ and $X_i$ are non-adjacent.

Assume, for the sake of contradiction, that $B \notin \MB{X_i \mid \mathcal{O}}$.
There exists a path $X_i \rightarrow A \leftarrow B$ in $\mathcal{G}$, and by Lemma~\ref{np}, $A \in \MB{X_i \mid \mathcal{O}}$.
Thus, the set $\MB{X_i \mid \mathcal{O}}$ does not block the path $X_i \rightarrow A \leftarrow B$, which implies that $X_i$ and $B$ are not d-separated by $\MB{X_i \mid \mathcal{O}}$ (in $\mathcal{G}$).

Consequently, by Assumption~\ref{ass:3} (Faithfulness), we have
\begin{equation*}
X_i \not\!\perp\!\!\!\perp B \mid \MB{X_i \mid \mathcal{O}}.
\end{equation*}
This contradicts the definition of MB (Definition~\ref{def:mb}). Hence, $B \in \MB{X_i \mid \mathcal{O}}$.

\end{proof}


\begin{lemma}[Sufficient Condition for v-structure Exclusion]\label{suff}
Suppose $X_i$ and $Y_j$ in $\mathcal{G}$ form exactly $K$ v-structures $X_i \rightarrow Z_k \leftarrow Y_j$ for $k=1, \dots, K$. 
Let $\mathcal{Z}^* = \{Z_1, \dots, Z_K\} \cup \bigcup_{k=1}^{K} \mathrm{De}(Z_k)$, where $\mathrm{De}(Z_k)$ denotes the set of all descendants of $Z_k$ ($X_i, Y_j \notin \mathcal{Z}^*$ by acyclicity of $\mathcal{G}$).
Then $Y_j \notin \MB{X_i \mid \mathcal{O}_{i \setminus \mathcal{Z}^*}}$, where $\mathcal{O}_{i \setminus \mathcal{Z}^*} = (\MB{X_i \mid \mathcal{V}} \cup \{X_i\}) \setminus \mathcal{Z}^*$.
\end{lemma}

\begin{proof}
Since $Y_j \in \spn{X_i}$, Lemma~\ref{lemma:mb} yields $Y_j \in \MB{X_i \mid \mathcal{V}}$. 
Combined with $Y_j \notin \mathcal{Z}^*$, it follows that $Y_j \in (\MB{X_i \mid \mathcal{V}} \cup \{X_i\}) \setminus \mathcal{Z}^* = \mathcal{O}_{i \setminus \mathcal{Z}^*}$.
Consequently, $Y_j \notin \MB{X_i \mid \mathcal{O}_{i \setminus \mathcal{Z}^*}}$ does not trivially hold.
We now prove it holds by contradiction.

Assuming $Y_j \in \MB{X_i \mid \mathcal{O}_{i \setminus \mathcal{Z}^*}}$, i.e., $Y_j \notin \mathbf{R}^*$, we first show $X_i$ and $Y_j$ are d-separated (in $\mathcal{G}$) by $\MB{X_i \mid \mathcal{O}_{i \setminus \mathcal{Z}^*}} \setminus \{Y_j\}$.
We proceed by considering all possible paths (in $\mathcal{G}$) between $X_i$ and $Y_j$, and showing each is blocked.

\begin{itemize}[topsep=0pt, left=2pt, itemsep=0pt]
\item \textit{\textbf{First-order paths.}} 
By the definition of v-structures, $X_i$ and $Y_j$ are non-adjacent, so no path of length 1 exists.

\item \textit{\textbf{Second-order paths.}} 
Any length-2 path between $X_i$ and $Y_j$ takes one of two forms:
\begin{itemize}[topsep=0pt, left=2pt, itemsep=2pt]
\item v-structure path $X_i \rightarrow Z_k \leftarrow Y_j$: 
The collider node $Z_k$ and all its descendants belong to $\mathcal{Z}^*$ and are therefore not in $\mathcal{O}_{i \setminus \mathcal{Z}^*}$.
Consequently, this path will always remain blocked given $\MB{X_i \mid \mathcal{O}_{i \setminus \mathcal{Z}^*}} \setminus \{Y_j\}$, as $Z_k$ and all its descendants are absent from this set.

\item Non-v-structure paths (i.e., $X_i \rightarrow A \rightarrow Y_j$, $X_i \leftarrow A \leftarrow Y_j$, or $X_i \leftarrow A \rightarrow Y_j$): 
\begin{itemize}[topsep=1pt, left=2pt, itemsep=2pt]
\item If $A \in \mathcal{Z}^*$: 
There must exist directed paths from both $X_i$ and $Y_j$ to $A$.
Since the graph $\mathcal{G}$ is acyclic, these directed paths preclude the existence of the reverse edges $A \rightarrow X_i$ and $A \rightarrow Y_j$ in $\mathcal{E}$.
Therefore, this case is impossible.

\item If $A \notin \mathcal{Z}^*$:
Combined with $A \in \MB{X_i \mid \mathcal{V}}$ (Lemma~\ref{lemma:mb}, as $A$ and $X_i$ are adjacent), it yields $A \in \mathcal{O}_{i \setminus \mathcal{Z}^*}$.
By Lemma~\ref{np} we get $A \in \MB{X_i \mid \mathcal{O}_{i \setminus \mathcal{Z}^*}}$, which means this path is blocked by $\MB{X_i \mid \mathcal{O}_{i \setminus \mathcal{Z}^*}} \setminus \{Y_j\}$.
\end{itemize}
\end{itemize}

\item \textit{\textbf{Higher-order paths.}}
Every path of length $\geq 3$ begins with a 2-edge prefix adjacent to $X_i$. We classify by the structure of this prefix:
\begin{itemize}[topsep=0pt, left=2pt, itemsep=2pt]
\item Prefix $X_i \rightarrow A \leftarrow B$:
\begin{itemize}[topsep=1pt, left=2pt, itemsep=2pt]
\item If $A \in \mathcal{Z}^*$:
We have $\mathrm{De}(A) \subseteq \bigcup_{k=1}^{K} \mathrm{De}(Z_k)$, thus $A$ and all its descendants belong to $\mathcal{Z}^*$ and are therefore not in $\mathcal{O}_{i \setminus \mathcal{Z}^*}$. Therefore, this path will always remain blocked given $\MB{X_i \mid \mathcal{O}_{i \setminus \mathcal{Z}^*}} \setminus \{Y_j\}$, as $A$ and all its descendants are absent from this set.

\item If $A \notin \mathcal{Z}^*$:
\begin{itemize}[topsep=1pt, left=2pt, itemsep=1pt]
\item If $B \in \mathcal{Z}^*$:
We have $\mathrm{De}(B) \subseteq \bigcup_{k=1}^{K} \mathrm{De}(Z_k)$, and since $A \in \mathrm{De}(B)$, we get that $A \in \bigcup_{k=1}^{K} \mathrm{De}(Z_k)$.
Therefore, this case is impossible.

\item If $B \notin \mathcal{Z}^*$:
We have $A, B \notin \mathcal{Z}^*$.
Combined with $A, B \in \MB{X_i \mid \mathcal{V}}$ (Lemma~\ref{lemma:mb}, as $A \in \ch{X_i}$ and $B \in \spn{X_i}$), it follows that $A, B \in \mathcal{O}_{i \setminus \mathcal{Z}^*}$.
By Lemma~\ref{sp}, this yields $B \in \MB{X_i \mid \mathcal{O}_{i \setminus \mathcal{Z}^*}}$.
Moreover, since the edge $B \rightarrow A$ exists, $B$ is not a collider on this path, so the path is blocked by $\MB{X_i \mid \mathcal{O}_{i \setminus \mathcal{Z}^*}} \setminus \{Y_j\}$. 

\end{itemize}
\end{itemize}

\item Prefix $X_i \rightarrow A \rightarrow B$: 
\begin{itemize}[topsep=1pt, left=2pt, itemsep=2pt]
\item If $A \in \mathcal{Z}^*$:

\begin{itemize}[topsep=1pt, left=2pt, itemsep=1pt]
\item If $B \in \mathcal{Z}^*$:
Because $A \in \mathcal{Z}^*$, there must exist directed paths from $Y_j$ to $A$.
Since the graph $\mathcal{G}$ is acyclic, these directed paths preclude the reverse directed paths from $A$ to $Y_j$, which means there must exist a collider node $C$ on this path such that $C \in \mathrm{De}(A) \subseteq \bigcup_{k=1}^{K} \mathrm{De}(Z_k)$.
Thus, $C$ and all its descendants belong to $\mathcal{Z}^*$ and are therefore not in $\mathcal{O}_{i \setminus \mathcal{Z}^*}$.
Consequently, this path will always remain blocked given $\MB{X_i \mid \mathcal{O}_{i \setminus \mathcal{Z}^*}} \setminus \{Y_j\}$, as $C$ and all its descendants are absent from this set.

\item If $B \notin \mathcal{Z}^*$:
We have $\mathrm{De}(A) \subseteq \bigcup_{k=1}^{K} \mathrm{De}(Z_k)$, and since $B \in \mathrm{De}(A)$, we get that $B \in \bigcup_{k=1}^{K} \mathrm{De}(Z_k)$.
Therefore, this case is impossible.
\end{itemize}

\item If $A \notin \mathcal{Z}^*$:
Combined with $A \in \MB{X_i \mid \mathcal{V}}$ (Lemma~\ref{lemma:mb}, as $A$ and $X_i$ are adjacent), it yields $A \in \mathcal{O}_{i \setminus \mathcal{Z}^*}$.
By Lemma~\ref{np} we get $A \in \MB{X_i \mid \mathcal{O}_{i \setminus \mathcal{Z}^*}}$, which means this path is blocked by $\MB{X_i \mid \mathcal{O}_{i \setminus \mathcal{Z}^*}} \setminus \{Y_j\}$.

\end{itemize}

\item Prefix $X_i \leftarrow A \leftarrow B$ or $X_i \leftarrow A \rightarrow B$: 
\begin{itemize}[topsep=1pt, left=2pt, itemsep=2pt]
\item If $A \in \mathcal{Z}^*$:
There must exist directed paths from $X_i$ to $A$.
Since the graph $\mathcal{G}$ is acyclic, these directed paths preclude the existence of the reverse edges $A \rightarrow X_i$ in $\mathcal{E}$.
Therefore, this case is impossible.

\item If $A \notin \mathcal{Z}^*$:
Combined with $A \in \MB{X_i \mid \mathcal{V}}$ (Lemma~\ref{lemma:mb}, as $A$ and $X_i$ are adjacent), it yields $A \in \mathcal{O}_{i \setminus \mathcal{Z}^*}$.
By Lemma~\ref{np} we get $A \in \MB{X_i \mid \mathcal{O}_{i \setminus \mathcal{Z}^*}}$, which means this path is blocked by $\MB{X_i \mid \mathcal{O}_{i \setminus \mathcal{Z}^*}} \setminus \{Y_j\}$.
\end{itemize}

\end{itemize}

\end{itemize}

Since every path between $X_i$ and $Y_j$ is blocked given $\MB{X_i \mid \mathcal{O}_{i \setminus \mathcal{Z}^*}} \setminus \{Y_j\}$, the d-separation holds. Consequently, under Assumption~\ref{ass:3} (Faithfulness), we have
\begin{equation}
X_i \perp\!\!\!\perp Y_j \mid \bigl( \MB{X_i \mid \mathcal{O}_{i \setminus \mathcal{Z}^*}} \setminus \{Y_j\} \bigr)
\label{eq:indep2}
\end{equation}
Meanwhile, by the definition of MB (Definition~\ref{def:mb}), we have
\begin{equation}
X_i \perp\!\!\!\perp \mathbf{R}^* \mid \MB{X_i \mid \mathcal{O}_{i \setminus \mathcal{Z}^*}},
\label{eq:mbdef2}
\end{equation}
where $\mathbf{R}^* = \mathcal{O}_{i \setminus \mathcal{Z}^*} \setminus (\MB{X_i \mid \mathcal{O}_{i \setminus \mathcal{Z}^*}} \cup \{X_i\} )$. By the contraction axiom (Lemma~\ref{lemma:contraction}), combining equations~\eqref{eq:indep2}~and~\eqref{eq:mbdef2} yields:
\begin{equation*}
X_i \perp\!\!\!\perp (\mathbf{R}^* \cup \{Y_j\}) \mid \bigl( \MB{X_i \mid \mathcal{O}_{i \setminus \mathcal{Z}^*}} \setminus \{Y_j\} \bigr),
\end{equation*}
which is a contradiction to the minimality of MB (Definition~\ref{def:mb}). Thus, our assumption is false, i.e., $Y_j \notin \MB{X_i \mid \mathcal{O}_{i \setminus \mathcal{Z}^*}}$.

\end{proof}


\begin{corollary}[Existence of v-structure Exclusion]\label{ext}
For any v-structure $X_i \rightarrow Z' \leftarrow Y_j$ in $\mathcal{G}$, there exists $\mathcal{Z}' \subseteq \MB{X_i \mid \mathcal{V}} \setminus \{Y_j\}$ such that $Y_j \notin \MB{X_i \mid \mathcal{O}_{i \setminus \mathcal{Z}'}}$ and $Z' \in \mathcal{Z}'$, where $\mathcal{O}_{i \setminus \mathcal{Z}'} = (\MB{X_i \mid \mathcal{V}} \cup \{X_i\}) \setminus \mathcal{Z}'$.
\end{corollary}

\begin{proof}
We have that $X_i$ and $Y_j$ form one or more v-structures in $\mathcal{G}$. 
Denoting all of them as $X_i \rightarrow Z_k \leftarrow Y_j$ for $k=1, \dots, K$, we have that $Z' \in \{Z_1, \dots, Z_K\}$.
Let $\mathcal{Z}^* = \{Z_1, \dots, Z_K\} \cup \bigcup_{k=1}^{K} \mathrm{De}(Z_k)$, where $\mathrm{De}(Z_k)$ denotes the set of all descendants of $Z_k$. Then, by Lemma~\ref{suff}, we immediately have $Y_j \notin \MB{X_i \mid \mathcal{O}_{i \setminus \mathcal{Z}^*}}$.

Let $\mathcal{Z}' = \mathcal{Z}^* \cap \MB{X_i \mid \mathcal{V}}$. By Lemma~\ref{lemma:mb}, $Z', Y_j \in \MB{X_i \mid \mathcal{V}}$ since $Z' \in \ch{X_i}$ and $Y_j \in \spn{X_i}$.
Thus, we have $Z' \in \mathcal{Z}'$.
Since $\mathcal{G}$ is acyclic, $X_i$ and $Y_j$ cannot be descendants of any $Z_k$. We have $X_i, Y_j \notin \bigcup_{k=1}^{K} \mathrm{De}(Z_k)$, which implies $X_i, Y_j \notin \mathcal{Z}^*$. 
Thus, by $\mathcal{Z}' = \mathcal{Z}^* \cap \MB{X_i \mid \mathcal{V}}$, we get $Y_j \notin \mathcal{Z}'$ and $\mathcal{Z}' \subseteq \MB{X_i \mid \mathcal{V}} \setminus \{Y_j\}$.

Meanwhile, by definition,
\begin{equation*}
\begin{aligned}
\mathcal{O}_{i \setminus \mathcal{Z}'}
&= (\MB{X_i \mid \mathcal{V}} \cup \{X_i\}) \setminus \mathcal{Z}' \\
&= (\MB{X_i \mid \mathcal{V}} \cup \{X_i\}) \setminus (\mathcal{Z}^* \cap \MB{X_i \mid \mathcal{V}}) \\
&= (\MB{X_i \mid \mathcal{V}} \cup \{X_i\}) \setminus \mathcal{Z}^* \\
&= \mathcal{O}_{i \setminus \mathcal{Z}^*}.
\end{aligned}
\end{equation*}
Therefore, it follows that $\MB{X_i \mid \mathcal{O}_{i \setminus \mathcal{Z}^*}} = \MB{X_i \mid \mathcal{O}_{i \setminus \mathcal{Z}'}}$, which means $Y_j \notin \MB{X_i \mid \mathcal{O}_{i \setminus \mathcal{Z}'}}$.
Consequently, for any such v-structure $X_i \rightarrow Z' \leftarrow Y_j$ in $\mathcal{G}$, there exists $\mathcal{Z}' \subseteq \MB{X_i \mid \mathcal{V}} \setminus \{Y_j\}$ such that $Y_j \notin \MB{X_i \mid \mathcal{O}_{i \setminus \mathcal{Z}'}}$ and $Z' \in \mathcal{Z}'$.
\end{proof}


\begin{lemma}[Soundness and Completeness of Skeleton]\label{sck}
Algorithm~\ref{alg:dcd-c} (DCD-C) adds the undirected edge $X_t - W$ to $\hat{\mathcal{E}}$ if and only if $X_t$ and $W$ are adjacent in $\mathcal{G}$.
\end{lemma}

\begin{proof}
We prove the soundness and completeness separately given any $X_t \in \mathcal{V}$. Since we assume the MB identifier in Algorithm~\ref{alg:dcd-c} operates as an oracle, we directly use the ground-truth set operator $\MB{\cdot \mid \cdot}$ throughout the proof.

According to Algorithm~\ref{alg:dcd-c}, the undirected edge $X_t - W$ is added to $\hat{\mathcal{E}}$ if and only if $W \in \mathcal{M}_t$ and there does not exist a subset $\mathcal{Z} \subseteq \mathcal{M}_t \setminus \{W\}$ such that $W \notin \MB{X_t \mid \mathcal{O}_{t \setminus \mathcal{Z}}}$, where $\mathcal{O}_{t \setminus \mathcal{Z}} = (\mathcal{M}_t \cup \{X_t\}) \setminus \mathcal{Z}$ and $\mathcal{M}_t = \MB{X_t \mid \mathcal{V}}$. 

\paragraph{\textit{Soundness.}}
To establish soundness, we show that if $W \in \mathcal{M}_t$ and there does not exist any subset $\mathcal{Z} \subseteq \mathcal{M}_t \setminus \{W\}$ such that $W \notin \MB{X_t \mid \mathcal{O}_{t \setminus \mathcal{Z}}}$, then $X_t$ and $W$ must be adjacent in the true graph $\mathcal{G}$. 

Equivalently, we prove the contrapositive: if $W \in \mathcal{M}_t$ but is not adjacent to $X_t$ in $\mathcal{G}$, then there exists a subset $\mathcal{Z} \subseteq \mathcal{M}_t \setminus \{W\}$ such that $W \notin \MB{X_t \mid \mathcal{O}_{t \setminus \mathcal{Z}}}$.

Since $X_t$ and $W$ are non-adjacent in $\mathcal{G}$ (i.e., $W \notin \pa{X_t} \cup \ch{X_t}$) and $W \in \MB{X_t \mid \mathcal{V}} = \mathcal{M}_t$, by Lemma~\ref{lemma:mb}, it follows that $W \in \spn{X_t}$.
Hence, $X_t$ and $W$ must form one or more v-structures in $\mathcal{G}$. 
Therefore, by Corollary~\ref{ext}, there exists a set $\mathcal{Z} \subseteq \mathcal{M}_t \setminus \{W\}$ such that $W \notin \MB{X_t \mid \mathcal{O}_{t \setminus \mathcal{Z}}}$, thereby establishing soundness.

\paragraph{\textit{Completeness.}}
To establish completeness, we show that if $X_t$ and $W$ are adjacent in the true graph $\mathcal{G}$, then $W \in \mathcal{M}_t$ and there does not exist any subset $\mathcal{Z} \subseteq \mathcal{M}_t \setminus \{W\}$ such that $W \notin \MB{X_t \mid \mathcal{O}_{t \setminus \mathcal{Z}}}$. 

Since $X_t$ and $W$ are adjacent in $\mathcal{G}$ (i.e., $W \in \pa{X_t} \cup \ch{X_t}$), by Lemma~\ref{lemma:mb}, it follows that $W \in \MB{X_t \mid \mathcal{V}} = \mathcal{M}_t$.
Hence, for any $\mathcal{Z} \subseteq \mathcal{M}_t \setminus \{W\}$, we have $W \in \mathcal{O}_{t \setminus \mathcal{Z}} = (\mathcal{M}_t \cup \{X_t\}) \setminus \mathcal{Z}$.
By Lemma~\ref{np}, it follows that $W \in \MB{X_t \mid \mathcal{O}_{t \setminus \mathcal{Z}}}$ for all such $\mathcal{Z}$.
Consequently, there does not exist any $\mathcal{Z} \subseteq \mathcal{M}_t \setminus \{W\}$ such that $W \notin \MB{X_t \mid \mathcal{O}_{t \setminus \mathcal{Z}}}$, thereby establishing completeness.

\end{proof}


\begin{lemma}[Soundness and Completeness of v-structure]\label{scv}
Algorithm~\ref{alg:dcd-c} (DCD-C) orients the v-structure $X_t \rightarrow Z \leftarrow Y$ in $\hat{\mathcal{E}}$ if and only if the v-structure $X_t \rightarrow Z \leftarrow Y$ exists in $\mathcal{G}$.
\end{lemma}

\begin{proof}
We prove the soundness and completeness separately given any $X_t \in \mathcal{V}$. Since we assume the MB identifier in Algorithm~\ref{alg:dcd-c} operates as an oracle, we directly use the ground-truth set operator $\MB{\cdot \mid \cdot}$ throughout the proof.

According to Algorithm~\ref{alg:dcd-c}, the orientation of v-structures occurs after all undirected edges have been added. 
At this stage, by Lemma~\ref{sck}, the skeleton of $\hat{\mathcal{G}} = (\mathcal{V}, \hat{\mathcal{E}})$ is identical to the true graph $\mathcal{G}$.
The v-structure $X_t \rightarrow Z \leftarrow Y$ is oriented in $\hat{\mathcal{E}}$ if and only if all of the following conditions are satisfied:
\begin{itemize}[topsep=0pt, left=8pt, itemsep=0pt]
\item $Y \in \mathcal{M}_t$.
\item There exists $\mathcal{Z} \subseteq \mathcal{M}_t \setminus \{Y\}$ such that $Y \notin \MB{X_t \mid \mathcal{O}_{t \setminus \mathcal{Z}}}$, where $\mathcal{O}_{t \setminus \mathcal{Z}} = (\mathcal{M}_t \cup \{X_t\}) \setminus \mathcal{Z}$.
\item $Z \in \mathcal{Z}$.
\item $Z$ is adjacent to both $X_t$ and $Y$ in $\hat{\mathcal{G}}$.
\item $X_t$ and $Y$ are not adjacent in $\hat{\mathcal{G}}$.
\end{itemize}

\paragraph{\textit{Soundness.}}
To establish soundness, we show that if all the conditions hold, there is a v-structure $X_t \rightarrow Z \leftarrow Y$ in $\mathcal{G}$.

We prove this by contradiction. 
Suppose there is no v-structure $X_t \rightarrow Z \leftarrow Y$ in $\mathcal{G}$.
Because the skeleton of $\hat{\mathcal{G}}$ is correctly identified, $Z$ is adjacent to both $X_t$ and $Y$, while $X_t$ and $Y$ are not adjacent in $\hat{\mathcal{G}}$.
Therefore, $\mathcal{G}$ must contain $X_t \leftarrow Z \rightarrow Y$ (or $X_t \rightarrow Z \rightarrow Y$, or $X_t \leftarrow Z \leftarrow Y$).

Since $Z \in \mathcal{Z}$, we have $Z \notin \mathcal{O}_{t \setminus \mathcal{Z}}$, and thus $Z \notin \MB{X_t \mid \mathcal{O}_{t \setminus \mathcal{Z}}}$.
Meanwhile, we already have $X_t, Y \notin \MB{X_t \mid \mathcal{O}_{t \setminus \mathcal{Z}}}$.
Therefore, the path $X_t \leftarrow Z \rightarrow Y$ (or $X_t \rightarrow Z \rightarrow Y$, or $X_t \leftarrow Z \leftarrow Y$) is not blocked by $\MB{X_t \mid \mathcal{O}_{t \setminus \mathcal{Z}}}$, which means $X_t$ and $Y$ are not d-separated by $\MB{X_t \mid \mathcal{O}_{t \setminus \mathcal{Z}}}$ (in $\mathcal{G}$).
By Assumption~\ref{ass:3} (Faithfulness), we have
\begin{equation*}
X_t \not\!\perp\!\!\!\perp Y \mid \MB{X_t \mid \mathcal{O}_{t \setminus \mathcal{Z}}},
\end{equation*}
which contradicts the definition of MB (Definition~\ref{def:mb}), given that $Y \in \mathcal{O}_{t \setminus \mathcal{Z}} = (\mathcal{M}_t \cup \{X_t\}) \setminus \mathcal{Z}$. Thus, our assumption is false, i.e., $\mathcal{G}$ must contain the v-structure $X_t \rightarrow Z \leftarrow Y$, thereby establishing soundness.

\paragraph{\textit{Completeness.}}
To establish completeness, we show that if there is a v-structure $X_t \rightarrow Z \leftarrow Y$ in $\mathcal{G}$, then all conditions are satisfied.

By Lemma~\ref{lemma:mb}, we have $Y \in \MB{X_t \mid \mathcal{V}} = \mathcal{M}_t$, as $Y \in \spn{X_t}$.
Meanwhile, since the skeleton of $\hat{\mathcal{G}}$ is correctly identified, $Z$ is adjacent to both $X_t$ and $Y$, while $X_t$ and $Y$ are not adjacent in $\hat{\mathcal{G}}$.

Meanwhile, by Corollary~\ref{ext}, there exists a subset $\mathcal{Z} \subseteq \mathcal{M}_t \setminus \{Y\}$ such that $Y \notin \MB{X_t \mid \mathcal{O}_{t \setminus \mathcal{Z}}}$ and $Z \in \mathcal{Z}$. 

Consequently, all conditions are satisfied, thereby establishing completeness.

\end{proof}


\begingroup
\renewcommand{\thetheorem}{\ref{the:AG}}
\begin{theorem}[Soundness and Completeness of DCD-C]
Under Assumptions~\ref{ass:1}, \ref{ass:2}, and~\ref{ass:3}, with an oracle MB identifier $\textup{MB}(\cdot \mid \cdot)$, Algorithm~\ref{alg:dcd-c} (DCD-C) returns the true CPDAG of $\mathcal{G}$.
\end{theorem}
\endgroup

\begin{proof}
By Lemmas~\ref{sck} and~\ref{scv}, Algorithm~\ref{alg:dcd-c} correctly identifies both the exact skeleton and v-structures.
According to~\citet{CPDAG}, the skeleton and v-structures uniquely determine a Markov equivalence class, which is represented by a CPDAG.
Based on the soundness and completeness of Meek rules~\citep{Meek}, Algorithm~\ref{alg:dcd-c} returns the true CPDAG.
This completes the proof.

\end{proof}

\section{Computational Complexity of DCD-C}\label{app:complexity}

We analyze the computational complexity of DCD-C (Algorithm~\ref{alg:dcd-c}) in terms of the number of MB identifications.
Given $\mathcal{G} = (\mathcal{V}, \mathcal{E})$ with $\mathcal{V} = \{X_1, \dots, X_d\}$ and $d = |\mathcal{V}|$, denote the maximum MB size in $\mathcal{G}$ by $m=\max_{X \in \mathcal{V}} |\MB{X \mid \mathcal{V}}|$ and let $m_i = |\MB{X_i \mid \mathcal{V}}|$ ($m_i \leq m$).
Assume that $\mathcal{E}\neq\emptyset$, so that $m>0$.

As shown in Algorithm~\ref{alg:dcd-c}, DCD-C identifies the MB of each variable once, and then iterates over all non-empty proper subsets of the identified MB.
Therefore, the total number of MB identifications is
\begin{equation*}
N_{\mathrm{MB}}
= d + \sum_{i:m_i>0}(2^{m_i}-2)
\leq d + d (2^{m}-2).
\end{equation*}
Hence, we have
\begin{equation*}
N_{\mathrm{MB}}
= O(d \cdot 2^{m}).
\end{equation*}

Consequently, the number of MB identifications required by DCD-C is controlled by both the number of variables and the maximum local structural complexity of the graph.

Compared with the complexity of constraint-based methods in terms of the number of CITs, DCD-C has a more favorable complexity of MB identifications.
Under the standard assumptions~\citep{SGS2} of causal sufficiency, the Markov condition, and faithfulness (Assumptions~\ref{ass:1},~\ref{ass:2} and~\ref{ass:3}), without imposing additional structural assumptions~\citep{recu2022}, \citet{recu2025} establish a theoretical lower bound for the number of CITs required by constraint-based methods:
\begin{equation*}
\Omega(d^2 + d\Delta_{\mathrm{in}}(\mathcal{G}) 2^{\Delta_{\mathrm{in}}(\mathcal{G})}).
\end{equation*}

For example, the classical PC algorithm~\citep{PC} has a worst-case complexity that scales exponentially with the maximum degree of the graph, requiring $O(d^{\Delta(\mathcal{G})+2})$ CITs, where $\Delta(\mathcal{G})$ denotes the maximum degree (both in-degree and out-degree) of any node in $\mathcal{G}$.

Moreover, in DCD-C, the results for different nodes do not affect each other and can therefore be computed in parallel.
In the ideal parallel case, the empirical runtime can be close to that required for $(2^m-1)$ MB identifications.
Meanwhile, the MB estimation procedure we use (Section~\ref{subsec:practical_e}) may be more efficient than some nonparametric CIT methods.
As a simple optimization problem, it can be solved on a single ordinary CPU core, and empirically the optimization speed is insensitive to the sample size $n$.
In contrast, the commonly used KCIT~\citep{KCIT} scales cubically with $n$, although some methods can approximate it in linear time~\citep{RCIT, ECIT}.

\section{Limitations and Future Directions}\label{app:future}

\paragraph{Decoupling Weights.}
Theoretically, while no single CIT is uniformly effective across all dependence structures~\citep{hardness, hardnessLocalPer}, 
satisfying Equation~\eqref{eq:Wperp} at the population level to guarantee MB identifiability is straightforward in our decoupling-based approach.

According to Definition~\ref{def:decoupled}, for any
$\tilde{P}_w(\mathbf{Z}) = \prod_{i} \tilde{P}_w(Z_i) = w(\mathbf{Z}) P(\mathbf{Z})$
under Assumption~\ref{ass:pos} (strict positivity), we can explicitly construct the weight as:
\begin{equation*}
w(\mathbf{Z}) = \frac{\prod_{i} \tilde{P}_w(Z_i)}{P(\mathbf{Z})}.
\end{equation*}
That is, for any desired collection of valid marginal distributions $\{\tilde P_w(Z_i)\}$, the corresponding density ratio directly defines a weighting function under which the joint distribution becomes a product of marginals.
For continuous $\mathbf{Z}$, there are infinitely many admissible choices of the marginal distributions $\{\tilde{P}_w(Z_i)\}$.
Each such choice induces a corresponding decoupling weight through the density ratio construction above. 
Therefore, the set of decoupling weights $\mathcal{W}_{\perp}(\mathbf{Z})$ can contain infinitely many distinct weights.

However, in practical estimation with finite samples (Section~\ref{subsec:practical_e}), we still face challenges:

\textit{Noise. }
Empirically, an extremely low noise level degrades the performance of DCD (Figure~\ref{fig:noise}), likely because learning effective decoupling weights becomes more challenging under such conditions.
Overly deterministic dependencies may limit the variation available in observational samples, making it difficult to decouple these relationships through sample weighting alone.
Future work could improve the current framework for learning weights and systematically investigate the interplay between the learning of decoupling weights and noise levels to further explore the potential of decoupling-based causal discovery.

\textit{High dimensionality. }
Although DCD is highly effective on small-scale graphs, its performance degrades as the graph size increases (Table~\ref{tab:main}), a limitation shared by most existing methods.
This is because causal discovery inherently requires measuring the relationships among multiple variables.
Similar to how large conditioning sets render CITs difficult, increasing the number of variables makes it challenging to achieve mutual independence via sample weighting (which is why we employ pairwise HSIC as an approximation in Section~\ref{subsec:practical_e}).
Fundamentally, the challenge of high dimensionality in causal discovery arises from the rapidly increasing sample complexity required to characterize distributional relationships as dimensionality grows.
Mitigating this issue represents a crucial direction for future exploration in learning decoupling weights, as well as in the broader field of causal discovery.

\textit{Sample selection. }
Practically, the weighting mechanism implicitly performs sample selection.
This provides an intuitive explanation for why, within the context of DAGs, non-adjacent spouse nodes maintain their dependence on the target after decoupling.
However, this selection process may assign substantially lower weights to certain sample points in the support, thereby potentially weakening the signal of dependencies and increasing the risk of false negatives. 
Learning decoupling weights that better preserve information is a promising direction for future work.

\paragraph{Graph Construction via MB Identification}

As a novel framework, Algorithm~\ref{alg:dcd-c} remains open to refinement and offers several directions for future development.

\textit{Large MB. }
As analyzed in Section~\ref{app:complexity}, the number of MB identifications required by Algorithm~\ref{alg:dcd-c} is governed by the size of the MBs. When the underlying graph is dense, both error accumulation and computational costs will increase.

\textit{Graph Topology. }
As illustrated in Figure~\ref{fig:graph}, DCD exhibits a pronounced advantage on SF graphs, whereas it slightly underperforms the constraint-based algorithm on more homogeneous ER graphs. 
On the one hand, the decoupling mechanism is less sensitive to variations in node degrees, maintaining robust performance across graph topologies. 
On the other hand, constraint-based methods can naturally benefit from homogeneous graphs due to lower-dimensional conditioning sets, leading to better performance on ER graphs than on SF graphs.

\textit{Combining Decoupling with Existing Strategies. }
Algorithm~\ref{alg:dcd-c} currently relies entirely on MB identification, and integrating CITs to complement this process represents a highly promising direction.
In fact, similar ideas rooted in CITs have been explored~\citep{GS1999, MB2008}.
A line of research~\citep{recu2021b, recu2021a, recu2022, recu2023, recu2025} performs MB identification before locally identifying removable variables.
A straightforward design would be to directly replace their CIT-based MB identification with our decoupling method.
Interestingly, despite originating from different motivations, our strategy exhibits complementarity with the idea of removable variables.
While they focus on preserving the stability of the projected MAG after variable removal, we focus on the structural disruption of the MB upon removal.
Overall, decoupling offers a direct MB perspective, providing a natural basis for integration with existing causal discovery paradigms.

\end{document}